\documentclass[twoside,11pt]{article}

\usepackage{thm-restate} 
\usepackage[preprint]{jmlr2e}

\usepackage{silence}
\usepackage[utf8]{inputenc} 
\usepackage[T1]{fontenc}    
\usepackage{url}            
\usepackage{booktabs}       
\usepackage{amsfonts}       
\usepackage{nicefrac}       
\usepackage{microtype}      
\usepackage[table]{xcolor}         
\usepackage{algorithm}
\usepackage{tikz}
\usepackage[noend]{algorithmic}
\usepackage{bm}
\usepackage{enumitem}
\usepackage{comment}
\usepackage{dsfont}
\usepackage{amsmath}
\usepackage{amssymb}
\usepackage{mathtools}
\usepackage[]{amsbsy}
\usepackage{commath}
\usepackage{physics} 
\usepackage{float} 
\usepackage{mathrsfs} 
\usepackage[caption=false,font=footnotesize]{subfig}
\usepackage{caption}

\usepackage[skip=0pt]{caption}

\usepackage{graphicx}
\graphicspath{{figures/}} 
\newtheorem{assumption}{Assumption}

\usepackage[capitalize,noabbrev]{cleveref}

\crefname{assumption}{Assumption}{Assumptions}
\crefname{theorem}{Theorem}{Theorems}
\crefname{equation}{}{}
\crefname{ALC@unique}{Line}{Lines}

\newcommand{\tikzcircle}[2][red,fill=red]{\tikz[baseline=-0.5ex]\draw[#1,radius=#2] (0,0) circle ;}%

\let\originalleft\left
\let\originalright\right
\renewcommand{\left}{\mathopen{}\mathclose\bgroup\originalleft}
\renewcommand{\right}{\aftergroup\egroup\originalright}

\newcommand{\paren}[1]{\left(#1\right)}
\newcommand{\bracket}[1]{\left[#1\right]}
\newcommand{\curly}[1]{\{#1\}}

\usepackage{etoolbox}
\newcounter{myalg}
\AtBeginEnvironment{algorithmic}{\refstepcounter{myalg}}
\makeatletter
\newcommand\thefontsize{The current font size is: \f@size pt}
\@addtoreset{ALC@unique}{myalg}
\makeatother

\definecolor{LightGray}{gray}{0.9}

\definecolor{MPL-blue}{HTML}{1F77B4}
\definecolor{MPL-orange}{HTML}{FF7F0E}
\definecolor{MPL-green}{HTML}{2CA02C}
\definecolor{MPL-red}{HTML}{D62728}
\definecolor{MPL-purple}{HTML}{9467BD}
\definecolor{MPL-brown}{HTML}{8C564B}
\definecolor{MPL-pink}{HTML}{E377C2}
\definecolor{MPL-gray}{HTML}{7F7F7F}
\definecolor{MPL-olive}{HTML}{BCBD22}
\definecolor{MPL-cyan}{HTML}{17BECF}

\DeclareMathOperator*{\argmax}{argmax} 
\DeclareMathOperator*{\argmin}{argmin}
\DeclareMathOperator{\E}{\mathbb{E}}
\DeclareMathOperator{\Prob}{\mathbb{P}}

\DeclareMathOperator{\R}{\mathbb{R}}

\newcommand{\eqdist}{\stackrel{\mathclap{D}}{=}}

\renewcommand{\paragraph}[1]{\textbf{#1}\hspace{1em}}

\usepackage{lastpage}
\jmlrheading{23}{2024}{1-\pageref{LastPage}}{7/23; Revised 12/23}{XXX}{XXX}{Carlos E. Luis, Alessandro G. Bottero, Julia Vinogradska, Felix Berkenkamp, and Jan Peters}

\ShortHeadings{Value-Distributional MBRL}{Luis, Bottero, Vinogradska, Berkenkamp, and Peters}
\firstpageno{1}

\begin{document}

\title{Value-Distributional Model-Based Reinforcement Learning}

\author{\name Carlos E. Luis \email carlosenrique.luisgoncalves@bosch.com \\
    \addr Bosch Corporate Research, TU Darmstadt\\
    Robert-Bosch-Campus 1, 71272 Renningen (Germany)
    \AND
    \name Alessandro G. Bottero \email alessandrogiacomo.bottero@bosch.com \\
    \addr Bosch Corporate Research, TU Darmstadt
    \AND
    \name Julia Vinogradska \email julia.vinogradska@bosch.com \\
    \addr Bosch Corporate Research
    \AND
    \name Felix Berkenkamp \email felix.berkenkamp@bosch.com \\
    \addr Bosch Corporate Research
    \AND
    \name Jan Peters \email jan.peters@tu-darmstadt.de \\
    \addr TU Darmstadt, German Research Center for AI (DFKI), Hessian.AI
    }

\editor{}

\maketitle

\begin{abstract}
    Quantifying uncertainty about a policy's long-term performance is important to solve
    sequential decision-making tasks. We study the problem from a model-based Bayesian
    reinforcement learning perspective, where the goal is to learn the posterior
    distribution over value functions induced by parameter (epistemic) uncertainty of
    the Markov decision process. Previous work restricts the analysis to a few moments
    of the distribution over values or imposes a particular distribution shape, e.g.,
    Gaussians. Inspired by distributional reinforcement learning, we introduce a Bellman
    operator whose fixed-point is the value distribution function. Based on our theory,
    we propose Epistemic Quantile-Regression (EQR), a model-based algorithm that learns
    a value distribution function. We combine EQR with soft actor-critic (SAC) for
    policy optimization with an arbitrary differentiable objective function of the
    learned value distribution. Evaluation across several continuous-control tasks shows
    performance benefits with respect to both model-based and model-free algorithms. The
    code is available at \url{https://github.com/boschresearch/dist-mbrl}.
\end{abstract}


\section{Introduction}
\label{sec:introduction}
Reinforcement learning (RL) tackles optimal decision-making in an unknown Markov Decision Process
(MDP) \citep{sutton_reinforcement_2018}. Uncertainty is at the heart of the RL problem: on one hand,
aleatoric uncertainty refers to the stochasticity in the MDP transitions and the RL agent's action
selection; on the other hand, \emph{epistemic} uncertainty appears due to lack of knowledge about
the MDP. During policy evaluation, both sources of uncertainty induce a distribution of possible
returns, which should be considered for policy optimization. For instance, in high-stakes
applications like medical treatments, accounting for aleatoric noise is key towards training
risk-averse policies \citep{chow_risk-sensitive_2015,keramati_being_2020}. Similarly, effective
exploration can be achieved by proper handling of epistemic uncertainty
\citep{deisenroth_pilco_2011,curi_efficient_2020}.

Two paradigms have emerged to capture uncertainty in the predicted outcomes of a policy. First,
\emph{distributional} RL \citep{bellemare_distributional_2017} models the aleatoric uncertainty
about returns, due to the inherent noise of the decision process. In contrast, \emph{Bayesian} RL
\citep{ghavamzadeh_bayesian_2015} captures the epistemic uncertainty about the unknown
\emph{expected} return of a policy, denoted as the \emph{value} function, due to incomplete
knowledge of the MDP. As such, the distribution over outcomes from each perspective has
fundamentally different meaning and utility. If we care about effective exploration of
\emph{unknown} (rather than stochastic) outcomes, then Bayesian RL is the appropriate choice of
framework \citep{osband_deep_2019}. 

In this paper, we focus on the Bayesian RL setting where a posterior distribution over possible MDPs
induces a distribution over value functions. The posterior over values naturally models the
epistemic uncertainty about the long-term performance of the agent, which is the guiding principle
behind provably-efficient exploration \citep{strehl_analysis_2008,jaksch_near-optimal_2010}. An open
question remains how to effectively model and learn the posterior distribution over value functions.
We approach this problem by using tools from distributional RL in the Bayesian framework. The key
idea is that, for time-inhomogeneous MDPs with a tabular representation, the value distribution
follows a Bellman equation from which we can derive an iterative estimation algorithm that resembles
methods from distributional RL. Based on this insight, we present a novel algorithm that uses a
learned value distribution for policy optimization.

\paragraph{Our contribution.}
We introduce the value-distributional Bellman equation that describes the relationship between the
value distributions over successive steps. Moreover, we show that the fixed-point of the associated
Bellman operator is precisely the posterior value distribution. Then, leveraging tools from
distributional RL, we propose a practical algorithm for learning the \emph{quantiles} of the value
distribution function. We propose Epistemic Quantile-Regression (EQR), a model-based policy
optimization algorithm that learns a distributional value function. Finally, we combine EQR with
soft actor-critic (SAC) to optimize a policy for any differentiable objective function of the
learned value distribution (e.g., mean, exponential risk, CVaR, etc.)

\subsection{Related work} 

\paragraph{Distributional RL.}
The treatment of the policy return as a random variable dates back to \citet{sobel_variance_1982},
where it is shown that the higher moments of the return obeys a Bellman equation. More recently,
distributional RL has emerged as a paradigm for modelling and utilizing the entire distribution of
returns \citep{tamar_temporal_2013,bellemare_distributional_2023}, with real-world applications
including guidance of stratospheric balloons \citep{bellemare_autonomous_2020} and super-human
race-driving in simulation \citep{wurman_outracing_2022}. The distributional RL toolbox has expanded
over the years with diverse distribution representations
\citep{bellemare_distributional_2017,dabney_distributional_2018,dabney_implicit_2018,yang_fully_2019}
and deeper theoretical understanding
\citep{bellemare_distributional_2019,rowland_analysis_2018,lyle_comparative_2019}. In our core
algorithm, we use quantile-regression (QR) by \citet{dabney_distributional_2018} as a tool for
learning the value, rather than return, distribution. Moreover, QR has been integrated with soft
actor-critic (SAC) \citep{haarnoja_soft_2018} for improved performance
\citep{wurman_outracing_2022,kuznetsov_controlling_2020}. At a high level, this paper combines model
learning with quantile-regression, which is then integrated with SAC for policy optimization.

\paragraph{Bayesian RL.}
Model-free approaches to Bayesian RL directly model the distribution over values, e.g., with
normal-gamma priors \citep{dearden_bayesian_1998}, Gaussian Processes \citep{engel_bayes_2003} or
ensembles of neural networks \citep{osband_deep_2016}. Instead, model-based Bayesian RL represents
uncertainty in the MDP dynamics, which must then be propagated to the value function. For
instance, the PILCO algorithm by \citet{deisenroth_pilco_2011} learns a Gaussian Process model
of the transition dynamics and integrates over the model's total uncertainty to obtain the expected
values. In order to scale to high-dimensional continuous-control problems, \citet{chua_deep_2018}
use ensembles of probabilistic neural networks (NNs) to capture both aleatoric and epistemic
uncertainty, as first proposed by \citet{lakshminarayanan_simple_2017}. Both approaches propagate
model uncertainty during policy evaluation and improve the policy via greedy exploitation over this
model-generated noise. 

Closest to our problem setting are approaches that explicitly model the value distribution function
or statistics thereof. The uncertainty Bellman equation (UBE) offers a framework to estimate the
\emph{variance} of the value distribution
\citep{odonoghue_uncertainty_2018,zhou_deep_2020,luis_model-based_2023}.
\citet{jorge_inferential_2020} propose a principled backwards induction framework to estimate value
distributions, with the caveat of assuming a Gaussian parameterization for practical implementation.
Perhaps closest to our approach is the work by \citet{dearden_model_1999}, which introduces a local
sampling scheme that maintains a sample-based approximation of the value distribution, which is
updated using a Bellman equation. While it does not assume a restrictive parametric form for the
distribution, it ignores that samples from the value distribution at successive states are
correlated through the Bellman equation; we make a similar approximation in our theory, see
\cref{sec:value_distribution_bellman_eq}. In our work, rather than generating random samples of the
value distribution, we keep track of a relevant set of statistics \citep{rowland_statistics_2019},
e.g., evenly spread quantiles, that have adequate coverage and representation power of the
underlying distribution.

\paragraph{Mixed Approaches.} Recent methods have combined distributional and
model-based RL methods. \citet{kastner_distributional_2023} introduce the distributional
model equivalence principle to train models that can plan optimally for risk-sensitive
objectives. \citet{moskovitz_tactical_2021,eriksson_sentinel_2022} aim to capture both
sources of uncertainty by  training an ensemble of return-distributional critics, where
each critic models aleatoric uncertainty and the ensemble recovers epistemic
uncertainty. Our approach is fundamentally different: we leverage tools from
distributional RL to model the epistemic uncertainty around \emph{expected} returns,
i.e., we average over aleatoric noise. Moreover, our experiments show that our value
representation with quantiles leads to substantial gains in performance over an ensemble
of critics.

\paragraph{Uncertainty-Aware Policy Optimization.} There exists a wide variety of policy
optimization objectives that leverage epistemic uncertainty. Multi-model MDPs (MMDPs)
\citep{steimle_multi-model_2021} consider a discrete distribution of MDPs and study the optimization
of the average value under the MDP uncertainty. Solving exactly for the optimal policy is only
possible for small MMDPs, but more recent methods can scale to larger problems
\citep{su_solving_2023}. Robust MDPs optimize for risk-averse objectives, like the percentile
criterion (also known as value-at-risk) \citep{delage_percentile_2010,behzadian_optimizing_2021}. In
practice, robust MDP objectives tend to be overly conservative, thus soft-robustness
\citep{derman_soft-robust_2018} has been proposed as an alternative objective, which is identical to
the risk-neutral objective of MMDPs.

In this paper we approach uncertainty-aware optimization from a different perspective. Instead
of fixing the policy optimization objective and designing a particular algorithm to solve for that
objective, we propose a general-purpose method that aims to optimize \emph{any} differentiable
function of a learned distribution of values. The strength of our approach is that it flexibly
accomodates an entire family of objectives that might suit different tasks, all within the same
algorithm and with minimal changes.

\section{Background \& Notation}
\label{sec:background}
In this section, we provide the relevant background and formally introduce the problem
of value distribution estimation. We use upper-case letters to denote random variables
and lower-case otherwise. The notation $\mathcal{P}(\mathcal{X})$ refers to the space of
probability distributions over the set $\mathcal{X}$, such that $\nu \in
\mathcal{P}(\mathcal{X})$ is a probability measure with the usual\footnote{Refers to the
power set $\sigma$-algebra for finite $\mathcal{X}$, the Borel $\sigma$-algebra for
infinite $\mathcal{X}$ and the product $\sigma$-algebra on products of such spaces
\citep{bellemare_distributional_2023}.} $\sigma$-algebra. We forego measure-theoretic
formalisms and further qualifications of measurability will be implied with respect to
the usual $\sigma$-algebra \citep[cf.][Remark 2.1]{bellemare_distributional_2023}.

\subsection{Markov Decision Processes}
We consider an agent that acts in an infinite-horizon MDP $\mathcal{M}
=\set{\mathcal{S},\mathcal{A},p, r, \gamma}$ with finite state space $\mathcal{S}$, finite action
space $\mathcal{A}$, unknown transition function $p: \mathcal{S} \times \mathcal{A} \to
\mathcal{P}(\mathcal{S})$ that maps states and actions to the set of probability distributions over
$\mathcal{S}$, a known\footnote{The theory results can be easily extended to unknown reward
functions.} and bounded reward function $r: \mathcal{S} \times \mathcal{A} \to \R$, and a discount
factor $\gamma \in [0,1)$. The agent is equipped with an action-selection stochastic policy $\pi:
\mathcal{S} \to \mathcal{P}(\mathcal{A})$ that defines the conditional probability distribution
$\pi(a\mid s)$, $(s,a) \in \mathcal{S} \times \mathcal{A}$. Given an initial state $s \in
\mathcal{S}$ and a policy $\pi$, the RL agent interacts with the environment and generates a
random \emph{trajectory} $T = \set{S_h, A_h, R_h}_{h=0}^\infty$, with $S_0 = s$ and for $h \geq 0$
we have $A_h \sim \pi(\cdot \mid S_h)$, $R_h = r(S_h, A_h)$, $S_{h+1} \sim p(\cdot \mid S_h, A_h)$.

\subsection{Return-Distributional Reinforcement Learning}
\label{subsec:return_dist_rl}
The \emph{return} of a policy, denoted $Z^\pi$, is a random variable defined as the discounted sum
of rewards along a trajectory, $Z^\pi(s) = \sum_{h=0}^\infty \bracket{\gamma^h R_h \mid S_0 = s}$. The randomness
in trajectories and returns originates from the inherent stochasticity of the environment dynamics
and the policy, oftentimes called \emph{aleatoric} uncertainty. A common objective for the RL agent
is to maximize the \emph{expected} return, where we average over this aleatoric noise to obtain a
deterministic function known as the \emph{value}. The value function of policy $\pi$ under dynamics
$p$, starting from $s \in \mathcal{S}$ is defined as a map $v^{\pi,p}: \mathcal{S} \to \R$ and is
given by
\begin{equation}
  \label{eq:standard_value}
  v^{\pi,p}(s) = \E_T\bracket{\sum_{h=0}^\infty \gamma^h R_h \middle| S_0=s, p},
\end{equation}
where we explicitly condition on the dynamics $p$; although redundant in the standard RL setting,
this notation will become convenient later when we consider a distribution over dynamics.

In contrast to learning the value function, \emph{return-distributional} RL aims to learn the entire
distribution of returns by leveraging the random variable \emph{return-distributional} Bellman
equation
\citep{bellemare_distributional_2017}
\begin{equation}
  \label{eq:return_dist_bellman}
  Z^\pi(s) \eqdist r(s, A) + \gamma Z^\pi(S'),
\end{equation}
where $A \sim \pi(\cdot \mid s)$, $S' \sim p(\cdot \mid s, A)$ and $(\eqdist)$ denotes equality in
distribution, i.e., the random variables in both sides of the equation may have different outcomes,
but they share the same probability distribution.

\subsection{Bayesian RL}
In this paper, we adopt a Bayesian perspective and model the unknown dynamics $p$ as a random
transition function $P$ with some prior distribution $\Phi(P)$. As the agent acts in $\mathcal{M}$,
it collects data\footnote{We omit time-step subscripts and refer to dataset $\mathcal{D}$ as
the collection of all available transition data.} $\mathcal{D}$ and obtains the posterior
distribution $\Phi(P \mid \mathcal{D})$ via Bayes' rule. More concretely, for tabular problems
we consider priors that admit analytical posterior updates (e.g., Dirichlet, Gaussian)
\citep{dearden_model_1999}, and for continuous state-action spaces we use neural network ensembles
\citep{lakshminarayanan_simple_2017} which have been linked to approximate Bayesian inference
\citep{osband_randomized_2018}.

\begin{figure}
  \includegraphics[width=\textwidth]{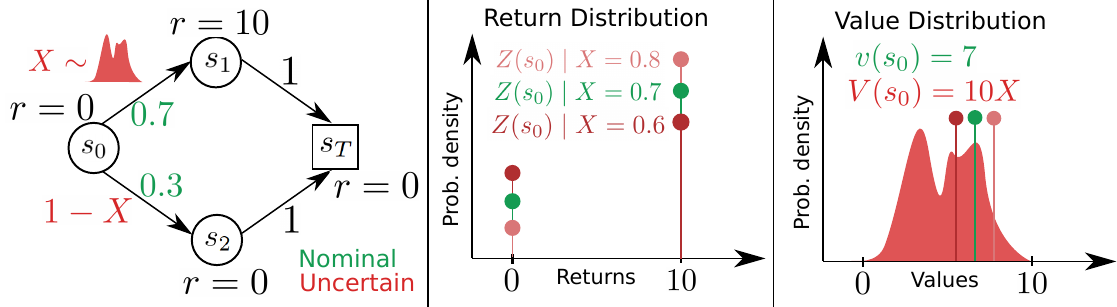}
  \caption{Return and value distributions in Bayesian RL. \textbf{(Left)} MDP with uncertain
  transition probability from $s_0$ given by a random variable $X \in [0, 1]$. \textbf{(Middle)}
  Return distributions at $s_0$ for realizations of $X$, including the nominal dynamics (green). The
  return distribution captures the \emph{aleatoric} noise under the sampled dynamics.
  \textbf{(Right)} Distribution of values at $s_0$. In the nominal case, the value $v(s_0)$ is a
  scalar obtained from averaging the aleatoric uncertainty of the return distribution $Z(s_0)$ under
  the nominal dynamics. In our setting, $V(s_0)$ is a random variable due to the \emph{epistemic}
  uncertainty around the MDP dynamics. To sample from $V(s_0)$ is equivalent to first sample $X =
  \tilde{x}$, compute the \emph{conditional} return distribution $Z(s_0) | X = \tilde{x}$ and
  finally average over the aleatoric noise.}
  \label{fig:return_vs_value_distribution}
\end{figure}

In what follows, we will assume $P \sim \Phi(P \mid \mathcal{D})$ and consider trajectories $T$
defined as previously but with next-states as $S_{h+1} \sim P(\cdot \mid S_h, A_h)$. Notably, the
sampling process of next states mixes two sources of uncertainty: the aleatoric noise, as with the
original MDP, but also the uncertainty in $P$ due to finite data, often called \emph{epistemic}
uncertainty. Consequently, the aleatoric and epistemic noise in trajectories propagates to the
returns. We define the value function of policy $\pi$ as a random variable under the random
dynamics $P$ as
\begin{equation}
  \label{eq:random_value}
  V^\pi(s) = v^{\pi, P}(s).
\end{equation}
According to the value function definition in \cref{eq:standard_value}, $V^\pi$ is an
expectation over the trajectories $T$ \emph{conditioned} on the random variable $P$, which means the
aleatoric noise of trajectories is averaged out, but the epistemic uncertainty (due to the
conditioning on $P$) remains and is propagated to $V^\pi$. Intuitively, to obtain a sample from
$V^\pi$ is equivalent to sample from the posterior $\Phi(P\mid \mathcal{D})$ and calculate the
corresponding expected return, i.e., the value. As such, the stochasticity of $V^\pi$ vanishes as we
gather data and $\Phi(P \mid \mathcal{D})$ concentrates around the true transition function $p$.

The main focus of this paper is to study the \emph{value-distribution}\footnote{We focus
on state-value functions for simplicity, but the results have a straightforward
extension for state-action-value functions.} function $\mu^\pi: \mathcal{S} \to
\mathcal{P}(\R)$, such that $V^\pi(s) \sim \mu^\pi(s)$, $\forall s \in \mathcal{S}$. As
such, $\mu^\pi$ represents the distribution of the \emph{epistemic} noise around the
\emph{expected} return of a policy. In \cref{fig:return_vs_value_distribution}, we
illustrate in a simple MDP the fundamental difference between return distributions and
value distributions: the former captures aleatoric noise from the decision process,
while the latter models \emph{epistemic} uncertainty stemming from uncertain MDP
dynamics. Refer to \cref{fig:example_value_distribution} for another example of an
uncertain transition probability and the value distribution it induces. While both value
and return distributions aim to obtain a richer representation of complex random
variables, only the former characterizes the type of uncertainty that is valuable for
effective exploration of the environment. 

\begin{figure}[t]
	\centering
  \subfloat{
    \includegraphics{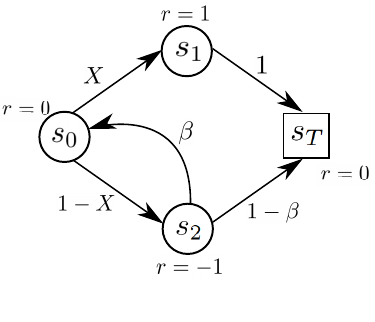}
  }
  \subfloat{
    \includegraphics{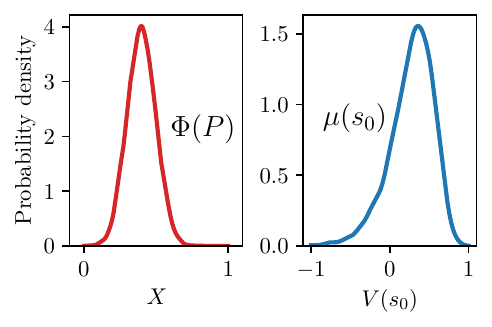}
  }
  \hfill
  \caption{Example value distribution. \textbf{(Left)} Uncertain MDP with a truncated Gaussian
  transition probability $X \sim \bar{\mathcal{N}}(\mu=0.4, \sigma=0.1)$ and a scalar
  (deterministic) $\beta \in [0, 1]$. For this example, we fixed $\beta=0.9$. \textbf{(Middle)}
  Distribution over MDPs, which corresponds directly to the distribution of $X$. \textbf{(Right)}
  Corresponding distribution of values for state $s_0$.}
  \label{fig:example_value_distribution}
\end{figure}

\section{The Value-Distributional Bellman Equation}
\label{sec:value_distribution_bellman_eq}
In this section, we establish the theoretical backbone of iterative algorithms for learning the
value-distribution function $\mu^\pi$. We include formal proofs in \cref{app:theory}.

For the nominal transition kernel $p$, we can relate the values at subsequent time steps using the
well-known Bellman equation
\begin{equation}
  \label{eq:standard_bellman}
  v^{\pi,p}(s) = \sum_a \pi(a\mid s)r(s,a) + \gamma \sum_{s', a} \pi(a \mid s) p(s' \mid s, a) v^\pi(s'),
\end{equation}
which holds for any policy $\pi$ and state $s \in \mathcal{S}$. The following statement is the
straightforward extension of \cref{eq:standard_bellman} in our Bayesian setting. While the result is
not novel, it serves as a building block towards our main theoretical contribution.

\begin{restatable}[Random Variable Value-Distribution Bellman Equation]{proposition}{randomvarbellman}
  \label{prop:rand_var_bellman}
  Let $V^\pi$ be the random value function defined in \cref{eq:random_value}. Then, it holds that   \begin{equation}
    \label{eq:distributional_bellman}
    V^\pi(s) = \sum_a \pi(a\mid s)r(s,a) + \gamma \sum_{s', a} \pi(a \mid s) P(s' \mid s, a)
    V^\pi(s'),
  \end{equation}
  for any policy $\pi$ and initial state $s \in \mathcal{S}$.
\end{restatable}

Note that the random variable value-distribution Bellman equation in
\cref{eq:distributional_bellman} differs from the random variable return-distribution Bellman
equation in \cref{eq:return_dist_bellman} in that the former holds in strict equality, while the
latter holds in the weaker notion of equality in distribution. The main caveat of
\cref{eq:distributional_bellman} with respect to model-free distributional RL is that, in general,
$P(s' \mid s,a)$ and $V^\pi(s')$ are correlated.

We now shift from discussing the random value function to focus on the value
distribution function. First, we provide a definition for $\mu^\pi$ that holds in
the general case. Intuitively, we can think of $\mu^\pi$ as the result of \emph{pushing}
the probability mass of the posterior distribution $\Phi$ through the map defined by the
value function \cref{eq:standard_value}. To formalize our statement, we leverage the
notion of pushforward measures akin to \citet{rowland_analysis_2018}.
\begin{definition}
  \label{def:pushforward}
  Given measurable spaces $\mathcal{X}$ and $\mathcal{Y}$, a measurable function $f:
  \mathcal{X} \to \mathcal{Y}$ and a measure $\nu \in \mathcal{P}(\mathcal{X})$, the
  pushforward measure $f_{\#}\nu \in \mathcal{P}(\mathcal{Y})$ is defined by
  $f_{\#}\nu(B) = \nu(f^{-1}(B))$, for all Borel sets $B \subseteq \mathcal{Y}$.
\end{definition}
Informally, given a random variable $X \sim \nu$ and the map $f$ as in
\cref{def:pushforward}, the pushforward of $\nu$ by $f$, denoted $f_{\#}\nu$, is defined
as the distribution of the random variable $f(X)$ \citep{bellemare_distributional_2023}.

With slight abuse of notation, define the map $v^\pi: \mathcal{S} \times \mathscr{P} \to
\R$ where $\mathscr{P}$ denotes the set of all transition functions\footnote{The
set of all transition functions can also be written as
$\mathcal{P}(\mathcal{S})^{\mathcal{S} \times \mathcal{A}}$ in standard set theory
notation.} for $(\mathcal{S}, \mathcal{A})$, such that $v^\pi(s, p) = v^{\pi, p}(s)$
for any $p \in \mathscr{P}$, $s \in \mathcal{S}$. Then, the value distribution function
$\mu^\pi$ is simply the pushforward measure of $\Phi$ by $v^\pi$.
\begin{definition}
  \label{def:value_as_pushforward}
  The value distribution function is defined as
  \begin{equation}
    \label{eq:value_pushforward}
    \mu^\pi = v^\pi_{\#}\Phi.
  \end{equation}
\end{definition}
From \cref{def:value_as_pushforward} we can already derive a simple (albeit, inefficient
and computationally expensive) sample-based algorithm to estimate $\mu^\pi$: sample from
the posterior $\Phi$ and compute the value function \cref{eq:standard_value} for each
sample, which results in samples from $\mu^\pi$. However, our main goal in this paper is
to find a recursive definition of $\mu^\pi$ such that we can introduce a simple, yet
efficient estimation algorithm.

The main challenge in establishing a recursion for learning $\mu^\pi$ is the dependency
between $P(s' \mid s,a)$ and $V^\pi(s')$ in \cref{eq:distributional_bellman}. We
side-step this issue by restricting our study to a family of MDPs under which these
random variables are independent, similar to previous work
\citep{odonoghue_uncertainty_2018, luis_model-based_2023}. All the results that follow
in this section hold under:
\begin{assumption}[Parameter Independence \citep{dearden_model_1999}]
    \label{assumption:transitions}
    The posterior over the random vector $P(\cdot \mid s, a)$ is independent for each
    pair $(s, a) \in \mathcal{S} \times \mathcal{A}$.
\end{assumption}
\begin{assumption}[Directed Acyclic MDP \citep{odonoghue_uncertainty_2018}]
    \label{assumption:acyclic}
    Let $\tilde{p} \in \mathscr{P}$ be a realization of the random variable $P$. Then,
    the MDP $\tilde{\mathcal{M}}$ with transition function $\tilde{p}$ is a directed
    acyclic graph, i.e., states are not visited more than once in any finite trajectory.
\end{assumption}
\begin{assumption}[Terminal State]
    \label{assumption:truncated}
    Define a terminal (absorbing) state $s_T$ such that $r(s_T, a) = 0$ and $p(s_T \mid
    s_T, a) = 1$ for any $a \in \mathcal{A}$ and $p \in \mathscr{P}$. Let $\tilde{p} \in
    \mathscr{P}$ be a realization of the random variable $P$. Then, the MDP
    $\tilde{\mathcal{M}}$ with transition function $\tilde{p}$ deterministically
    transitions to $s_T$ after a finite horizon $H \leq \abs{\mathcal{S}}$, i.e.,
    $\tilde{p}(s_T \mid s_H, a) = 1$ for any $a \in \mathcal{A}$.
\end{assumption}
The consequence of these assumptions is that $P(s' \mid s, a)$ and $V^\pi(s')$ are
conditionally independent for all triplets $(s, a, s')$ (see
\cref{lemma:independence}). \Cref{assumption:transitions} is satisfied when modelling
state transitions as independent categorical random variables for every pair $(s, a)$,
with the unknown parameter vector $P(\cdot \mid s, a)$ under a Dirichlet prior
\citep{dearden_model_1999}. \Cref{assumption:acyclic,assumption:truncated} imply
that any infinite trajectory is composed of \emph{distinct} states for the first $H$
steps and then remains at the terminal state $s_T$ indefinitely. Our theoretical results
do not hold in the general case of MDPs with cycles. However, one may still obtain
reasonable approximations by considering an equivalent time-inhomogeneous MDP without
cycles (also known as the ``unrolled'' MDP \citep{odonoghue_uncertainty_2018}) by
forcing a transition to a terminal state after $H$ steps (see \cref{app:unrolling} for
an example). The approximation improves as $H \to \infty$, but in the limit implies an
infinite state space, which would then require additional measure-theoretic
considerations outside the scope of this work \citep[cf.][Remark
2.3]{bellemare_distributional_2023}. Despite these limitations, in
\cref{sec:quantile_regression} we empirically show that the algorithm stemming from our
theory yields reasonable esimates of the value distribution $\mu^\pi$ in MDPs
\emph{with} cycles.

Beyond tabular representations of the transition function, introducing function
approximation violates \cref{assumption:transitions} due to generalization of the model
\citep{odonoghue_uncertainty_2018,zhou_deep_2020,derman_bayesian_2020}. However, in
\cref{sec:experiments} our approach demonstrates strong empirical performance when
paired with neural networks for function approximation.

We want to highlight that, under our assumptions, the \emph{mean} value function
$\E\bracket{V^\pi}$ corresponds exactly to the value function under the mean of the
posterior $\Phi$, denoted $\bar{p}$. That is, $\E\bracket{V^\pi} = v^{\pi, \bar{p}}$
\citep{luis_model-based_2023}. If our ultimate goal is to estimate the mean of
$\mu^\pi$, then standard approaches to approximately solve the Bellman expectation
equation suffice. However, in this paper we motivate the need for the distributional
approach as it allows to flexibly specify policy objectives \emph{beyond} maximizing the
mean of the values. For instance, in \cref{sec:experiments} we explore the performance
of optimistic value objectives.

To establish a Bellman-style recursion defining the value distribution function, we
use the notion of pushforward measure from \cref{def:pushforward}. In particular,
we are interested in the pushforward of the value distribution by the bootstrap function
in \cref{eq:distributional_bellman}. First, for any MDP with transition function $p
\in \mathscr{P}$, we denote by $p^\pi: \mathcal{S} \to \mathcal{P}(\mathcal{S})$ the
transition function of the Markov Reward Process (MRP) \emph{induced} by policy $\pi$,
defined by $p^\pi(s' \mid s) = \sum_a \pi(a \mid s) p(s' \mid s, a)$. Further, it is
convenient to adopt the matrix-vector notation of the standard Bellman equation:
$\mathbf{v}^{\pi,p} = \mathbf{r}^\pi + \gamma \mathbf{p}^\pi \mathbf{v}^{\pi,p}$, where
$\mathbf{r}^\pi \in \R^{\mathcal{S}}$, $\mathbf{v}^{\pi,p} \in \R^{\mathcal{S}}$ are
vectors and $\mathbf{p}^\pi \in \R_{[0, 1]}^{\mathcal{S} \times \mathcal{S}}$ is a
so-called stochastic matrix whose entries are restricted to $[0, 1]$ and whose rows sum
up to $1$, i.e., such that it represents the transition function $p^\pi$. Then, we
define the bootstrap function $b_{r,p,\gamma} : \R^{\mathcal{S}} \to \R^{\mathcal{S}}$
applied to value vectors:
\begin{equation}
  \label{eq:bootstrap_vector}
  b_{r,p,\gamma}: \mathbf{v} \to \mathbf{r} + \gamma \mathbf{p} \mathbf{v},
\end{equation}
for an arbitrary $\mathbf{r} \in \R^{\mathcal{S}}$, $\mathbf{p} \in \R_{[0,
1]}^{\mathcal{S} \times \mathcal{S}}$ and $\gamma \in [0, 1)$. Applying $b_{r,p,\gamma}$
is a combination of adding $\mathbf{r}$ to a $\gamma$-scaled linear transformation of
the input vector. Further, we express mixtures with weights given by the posterior
$\Phi(P \mid \mathcal{D})$ more compactly with the notation\footnote{Adapted from
\citet{bellemare_distributional_2023}. It refers to a mixture distribution and must not
be mistaken by an expected value, which is a scalar.} $\E_P[\cdot]$, where the argument
is a probability distribution depending on $P$. Given the pushforward and mixture
operations, we can now propose a Bellman equation for the value distribution function
$\mu^\pi$.
\begin{restatable}[Value-Distribution Bellman Equation]{lemma}{bellmandist}
  \label{lemma:bellman_dist_assumptions}
  The value distribution function $\mu^\pi$ obeys the Bellman equation.
  \begin{equation}
    \label{eq:distributional_bellman_assumptions}
    \mu^\pi = \E_{P}\bracket{(b_{r^\pi, P^\pi, \gamma})_{\#}\mu^\pi}
  \end{equation}
  for any policy $\pi$.
\end{restatable}
\Cref{lemma:bellman_dist_assumptions} provides the theoretical backbone towards designing an
iterative algorithm for learning the value distribution function. In particular, the recursive
definition for $\mu^\pi$, which corresponds to a mixture of pushforwards of itself, leads to
efficient estimation methods by dynamic programming. Alternatively, we can also write the
value-distributional Bellman equation for each state $s \in \mathcal{S}$. With slight abuse of
notation, define $b_{r, \gamma}: \R \to \R$ as the map $v \to r + \gamma v$, then
\begin{equation}
  \label{eq:distributional_bellman_per_state}
  \mu^\pi(s) = \E_P \bracket{
    \sum_{s'}P^\pi(s' \mid s) (b_{r^\pi(s), \gamma})_{\#} \mu^\pi(s')
  }.
\end{equation}
Note that \cref{eq:value_pushforward} holds generally, while
\cref{eq:distributional_bellman_assumptions,eq:distributional_bellman_per_state} only
hold under \cref{assumption:acyclic,assumption:transitions,assumption:truncated}.
Moreover, the operator $\E_P[\cdot]$ is well-defined in
\cref{eq:distributional_bellman_assumptions,eq:distributional_bellman_per_state} since
$P^\pi(s' \mid s)$ is bounded in $[0, 1]$ for all $s,s' \in \mathcal{S}$
\citep{billingsley_probability_1995}

In \cref{fig:bellman_diagram} we illustrate the core operations involved in the value-distributional
Bellman recursion prescribed by \cref{eq:distributional_bellman_per_state}.

\begin{figure}
  \includegraphics[width=\textwidth]{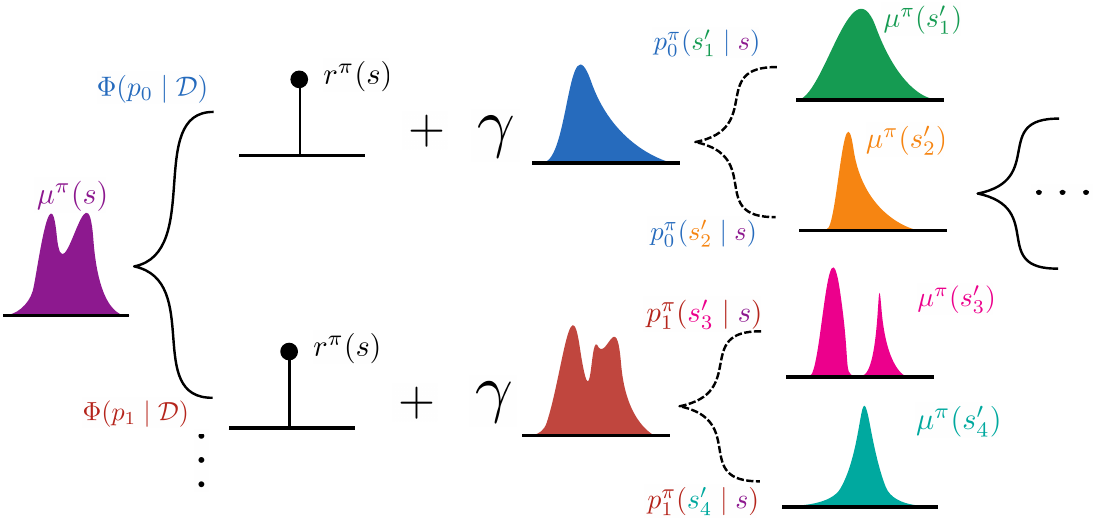}
  \caption{Visualization of the value-distributional Bellman backups, as prescribed by
  \cref{eq:distributional_bellman_per_state}. We identify four operations on
  distributions: infinite mixture over posterior transition functions (solid
  braces), shift by reward, scale by discount factor and mixture over next states
  (broken line braces)\protect\footnotemark. The main difference w.r.t the
  return-distributional backup \citep[cf.][Figure 2.6]{bellemare_distributional_2023} is
  the presence of the two distinct mixture operations.}
  \label{fig:bellman_diagram}
\end{figure}

\footnotetext{The pushforward operator $(b_{r, \gamma})_{\#}$ is linear \citep[cf.][Exercise
  2.13]{bellemare_distributional_2023}, so it can be moved outside the next-state mixture operation
  as depicted in the diagram.}

From \cref{eq:distributional_bellman_assumptions} we can extract an operator that acts on arbitrary
value distribution functions.
\begin{definition}
  The value-distributional Bellman operator $\mathcal{T}^\pi: \mathcal{P}(\R)^{\mathcal{S}} \to
  \mathcal{P}(\R)^{\mathcal{S}}$ is defined by 
  \begin{equation}
    \mathcal{T}^\pi \mu = \E_P\bracket{(b_{r^\pi, P^\pi, \gamma})_{\#}\mu}
  \end{equation}
\end{definition}
Intuitively, the operator $\mathcal{T}^\pi$ corresponds to mixing pushforward distributions, where
the pushforward itself involves shifting, scaling and linearly transforming the probability mass.
The natural question that follows is whether we can establish convergence to $\mu^\pi$ by repeated
applications of $\mathcal{T}^\pi$ starting from an arbitrary initial guess $\mu_0$.

Our convergence result is an adaptation of the standard distributional RL analysis by
\citet{bellemare_distributional_2023}. With some abuse of notation, we adopt the supremum
$p$-Wasserstein distance $\bar{w}_p$ to establish contractivity of the operator~$\mathcal{T}^\pi$
(see \cref{def:wasserstein} in \cref{app:theory}).
\begin{restatable}[]{theorem}{contraction}
  \label{thm:contraction}
  The operator $\mathcal{T}^\pi$ is a $\gamma$-contraction with respect to $\bar{w}_p$
  for all $p \in [1, \infty)$. That is, $\bar{w}_p(\mathcal{T}^\pi \mu, \mathcal{T}^\pi
  \mu') \leq \gamma \bar{w}_p(\mu, \mu')$ for all $\mu, \mu' \in
  \mathcal{P}(\R)^{\mathcal{S}}$ such that $V(s') \sim \mu(s')$, $V'(s') \sim
  \mu'(s')$ are conditionally independent of $P^\pi(s' \mid s)$ given $s' \in
  \mathcal{S}$.
\end{restatable}
\Cref{thm:contraction} parallels similar results in standard RL and model-free distributional RL,
in that it allows us to establish the convergence of iterated applications of $\mathcal{T}^\pi$ and
characterize the operator's fixed-point.

\begin{restatable}[]{corollary}{convergence}
  \label{corollary:convergence}
  Denote the space of value distribution functions with bounded support\footnote{Under
  bounded reward functions, the corresponding value distributions have bounded support.
  The corollary can be relaxed to distributions with bounded moments (see Proposition
  4.30 in \citet{bellemare_distributional_2023}.)} by $\mathcal{P}_B(\R)^\mathcal{S}$.
  Given an arbitrary value distribution function $\mu_0 \in
  \mathcal{P}_B(\R)^\mathcal{S}$, the sequence $\set{\mu_k}_{k=0}^\infty$ defined by
  $\mu_{k+1} = \mathcal{T}^\pi\mu_k$ for all $k \geq 0$ is such that $\bar{w}_p(\mu_k,
  \mu^\pi) \leq \gamma^k\bar{w}_p(\mu_0, \mu^\pi) \to 0$ as $k \to \infty$ for $p \in
  [1, \infty)$. That is, $\mu^\pi$ is the unique fixed-point of the operator
  $\mathcal{T}^\pi$.
\end{restatable}
\begin{proof}
  We wish to apply \cref{thm:contraction} on the sequence of pairs $\set{(\mu_k,
  \mu^\pi)}_{k=0}^\infty$. The conditional independence assumption required to apply
  \cref{thm:contraction} holds for $\mu^\pi$ (see \cref{lemma:independence}), but it is
  straightforward to show it also holds (under
  \cref{assumption:transitions,assumption:acyclic,assumption:truncated}) for all the
  elements of the sequence $\set{\mu_k}_{k=0}^\infty$ (see
  \cref{lemma:indep_property}). Further, since we consider bounded rewards, it follows
  immediately that $\mu^\pi \in \mathcal{P}_B(\R)^\mathcal{S}$. Moreover, it can be
  shown that the operator $\mathcal{T}^\pi$ maps $\mathcal{P}_B(\R)^\mathcal{S}$ onto
  itself, such that for arbitrary $\mu \in \mathcal{P}_B(\R)^\mathcal{S}$ then
  $\mathcal{T}^\pi\mu \in \mathcal{P}_B(\R)^\mathcal{S}$ (see
  \cref{lemma:bounded_support}). By \cref{thm:contraction}, $\mathcal{T}^\pi$ is then a
  \emph{contraction mapping} and by Banach's fixed-point theorem $\mathcal{T}^\pi$
  admits a unique fixed-point which is the limiting value of the sequence
  $\set{\mu_k}_{k=0}^\infty$. Since $\mu^\pi = \mathcal{T}^\pi \mu^\pi$ holds by
  \cref{lemma:bellman_dist_assumptions}, then $\mu^\pi$ must be the unique fixed-point
  of $\mathcal{T}^\pi$.
\end{proof}
In summary, \cref{corollary:convergence} establishes that repeated applications of $\mathcal{T}^\pi$
from an arbitrary initial guess converges to the value distribution function $\mu^\pi$. Inspired by
this theoretical result, in the remaining sections we introduce and evaluate a practical algorithm
for learning the value distribution function.

\section{Quantile-Regression for Value-Distribution Learning}
\label{sec:quantile_regression}
In the previous section we described an iterative process that converges to $\mu^\pi$
starting from an arbitrary value distribution with bounded support. In practice,
however, to implement such a recursion we must project the value distributions onto some
finite-dimensional parameter space. Following the success of quantile distributional RL
\citep{dabney_distributional_2018}, we adopt the quantile parameterization. Let
$\mathcal{V}_m$ be the space of quantile distributions with $m$ quantiles and
corresponding quantile levels $\tau_i = 1/m$ for $i = \set{1, \dots, m}$ and $\tau_0 =
0$. Define a parametric model $q: \mathcal{S} \to \R^m$, then the quantile distribution
$\mu_q \in \mathcal{V}_m$ maps states to a uniform probability distribution supported on
$q_i(s)$. That is, $\mu_q(s) = \frac{1}{m}\sum_{i=1}^m \delta_{q_i(s)}$, where
$\delta_x$ denotes the Dirac delta distribution centered at $x \in \R$, such that
$\mu_q(s)$ is a uniform mixture of Dirac deltas where the particle $q_i(s)$ corresponds
to the $\tau_i$-quantile at state $s$.  With this parameterization, our aim now becomes
to compute the so-called quantile projection of $\mu^\pi$ onto $\mathcal{V}_m$, given by
\begin{equation}
  \label{eq:argmin_wasserstein}
  \Pi_{w_1}\mu^\pi := \argmin_{\mu_q \in \mathcal{V}_m} w_1(\mu^\pi, \mu_q),
\end{equation}
where $w_1$ is the $1$-Wasserstein distance. Define $F^{-1}_{\mu^\pi}$ as the inverse cumulative
distribution function of $\mu^\pi$, then the distance metric becomes
\begin{equation}
  w_1(\mu^\pi, \mu_q) = \sum_{i = 1}^m \int_{\tau_{i-1}}^{\tau_i} \abs{F^{-1}_{\mu^\pi}(\omega) - q_i}d\omega,
\end{equation}
since $\mu_q$ is a uniform distribution over $m$ Dirac deltas with support $\set{q_1,
\dots, q_m}$. Let $\hat{\tau}_i = (2i - 1) / 2m$, then a valid minimizer of
\cref{eq:argmin_wasserstein} exists and is achieved by selecting $q_i =
F^{-1}_{\mu^\pi}(\hat{\tau}_i)$ \citep[cf.][Lemma 2]{dabney_distributional_2018}. In
summary, quantile projection as defined in \cref{eq:argmin_wasserstein} is equivalent
to estimating each $\hat{\tau}_i$-quantile of $\mu^\pi$.

\begin{figure}
  \includegraphics{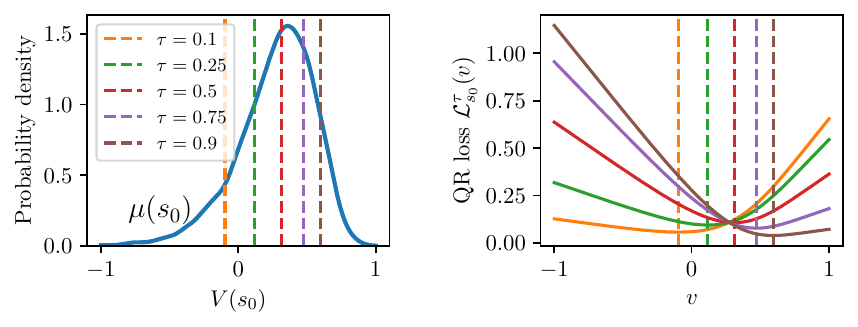}
  \caption{Quantile-regression loss for the example MDP of \Cref{fig:example_value_distribution}. \textbf{(Left)} Probability density of values for state $s_0$, with five quantile levels in colored vertical lines. \textbf{(Right)} The quantile regression loss \cref{eq:quantile_regression_loss} for the five quantile levels; the vertical lines correspond to the minimum of the color-matching loss. The vertical lines on both plots match upto numerical precision, meaning that following the gradient of such a convex loss function would indeed converge to the quantile projection $\Pi_{w_1}\mu$.}
  \label{fig:qr_loss}
\end{figure}

We follow closely the treatment by \citet{rowland_analysis_2023} of quantile-regression
temporal-difference learning for return-distributions and adapt it to instead work on
value-distributions. The following loss function corresponds to the quantile-regression
problem of estimating the $\tau$-quantile of the value distribution $\mu^\pi$:
\begin{equation}
  \label{eq:quantile_regression_loss}
  \mathcal{L}^{\tau, \pi}_s(v) = \E_P\bracket{
    \paren{
      \tau \mathds{1}\curly{V^\pi(s) > v} + (1 - \tau) \mathds{1} \curly{V^\pi(s) < v}
    } \abs{V^\pi(s) - v}
  }.
\end{equation}
It is an asymmetric convex loss function, where quantile overestimation and underestimation errors
are weighted by $\tau$ and $1-\tau$, respectively. The unique minimizer of this loss is the
$\tau$-quantile of $\mu^\pi$, which we illustrate with an example in \Cref{fig:qr_loss}.

Our goal is to propose a practical algorithm to learn the value distribution function based on the
quantile-regression loss \cref{eq:quantile_regression_loss}. If we have access to samples of
$V^\pi(s)$, denoted $\tilde{v}^\pi(s)$, then we can derive an unbiased estimate of the negative gradient of \cref{eq:quantile_regression_loss} and obtain the update rule
\begin{equation}
  q_i(s) \leftarrow q_i(s) + \alpha \paren{
    \tau_i - \mathds{1}\curly{\tilde{v}^\pi(s) < q_i(s)}
  },
\end{equation}
where $\alpha$ is some scalar step size. One option to sample $V^\pi=\tilde{v}^\pi$ would be to
first sample a model $P=\tilde{p}$ and then solve the corresponding Bellman equation. Instead, we
use a computationally cheaper alternative (albeit biased) and bootstrap like in temporal-difference
learning, so that the samples are defined as
\begin{equation}
  \tilde{v}^\pi(s) = r^\pi(s) + \gamma\sum_{s'}\tilde{p}^\pi(s' \mid s)q_{J}(s'),
\end{equation}
where $J \sim \textnormal{Uniform}(1, \dots, m)$. Lastly, we reduce the variance of the gradient
estimate by averaging over the values of $J$, which results in the update
\begin{equation}
  \label{eq:qtd_update}
  q_i(s) \leftarrow q_i(s) + \frac{\alpha}{m} \paren{
    \tau_i -\sum_{j=1}^m \mathds{1}\set{r^\pi(s) + \gamma \sum_{s'}\tilde{p}^\pi(s' \mid s) q_j(s') < q_i(s)}
  }.
\end{equation}

We introduce EQR in \Cref{algorithm:value_quantile_learning} to iteratively learn the value
distribution function $\mu^\pi$. From an arbitrary initial guess of quantiles, we sample an MDP from
the posterior and update the quantiles using \cref{eq:qtd_update} for all states until convergence.
The following examples illustrate the performance of EQR in tabular problems.
\begin{algorithm}[t]
   \caption{Epistemic Quantile-Regression (EQR)}
   \label{algorithm:value_quantile_learning}
\begin{algorithmic}[1]
  \STATE {\bfseries Input:} Posterior MDP $\Phi$, policy $\pi$, number of quantiles $m$.
  \STATE Randomly initialize estimates $\set{q_i(s)}_{i=1}^m$ for all $s \in \mathcal{S}$
  \REPEAT
    \STATE Sample $\tilde{p}$ from posterior $\Phi$
    \FOR{$i = 1, \dots, m$}
      \STATE Update $q_i(s)$ with \cref{eq:qtd_update} for all $s \in \mathcal{S}$
    \ENDFOR
  \UNTIL convergence
  \RETURN $\set{q_i(s)}_{i=1}^m$
\end{algorithmic}
\end{algorithm}

\begin{figure}[t]
  \includegraphics{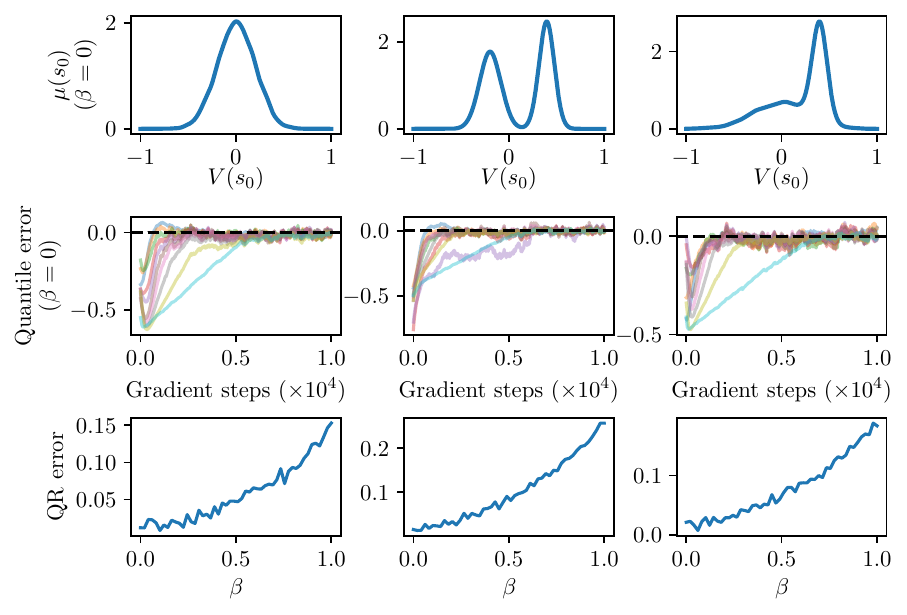}
  \caption{Performance of quantile-regression for value-distribution learning in the example MDP of
  \Cref{fig:example_value_distribution}. The parameter $\beta$ controls the covariance between
  $V(s_0)$ and $P(s_2 | s_0)$; the covariance increases with $\beta$ and is zero for $\beta=0$.
  \textbf{(Top)} Value distributions (Gaussian, bi-modal and heavy-tailed) generated by different
  prior distributions of the parameter $\delta$. \textbf{(Middle)} Evolution of the per-quantile
  estimation error $(\Pi_{w_1}\mu(s_0) - \mu_q(s_0))$ between the true quantile projection and
  the prediction; for $\beta=0$, our algorithm oscillates around the true quantile
  projection. \textbf{(Bottom)} $1$-Wasserstein metric between the true quantile projection and the
  estimate $\mu_q$ after $10^4$ gradient steps, as a function of the correlation parameter
  $\beta$. As $\beta$ moves from zero to one, the regression error increases and the algorithm no
  longer converges to the true quantiles, although the error is relatively small.}
  \label{fig:qr_examples}
\end{figure}

\begin{example}[Toy MDP]
  \label{ex:toy_mdp}
  Consider once more the tabular MDP of \cref{fig:example_value_distribution}. Our goal
  is to assess the convergence properties of EQR both when
  \cref{assumption:acyclic,assumption:transitions,assumption:truncated} hold, but
  also when they are violated \footnote{In order to be closer to standard settings,
  when the assumptions are violated (i.e., MDP contains cycles) we do not unroll the MDP
  as described in \cref{app:unrolling}.}. We control the degree of violation of
  \cref{assumption:acyclic,assumption:truncated} via the parameter $\beta$ of the MDP.
  If $\beta=0$, the assumptions hold and $V(s_0)$ and $P(s_2 | s_0)$ are decorrelated.
  As $\beta$ goes from zero to one, the covariance between these two random variables
  increases monotonically. We manually design three MDP posterior distributions that
  result in diverse distributions for $V(s_0)$. The value distributions shown in the top
  row of \cref{fig:qr_examples} are the result of modelling the MDP parameter $X$ as the
  following mixtures of truncated Gaussian distributions:
  $X_{\text{left}}\sim\bar{\mathcal{N}}(\mu=0.5, \sigma=0.1)$,
  $X_{\text{middle}}\sim0.5\bar{\mathcal{N}}(\mu=0.3, \sigma=0.03) +
  0.5\bar{\mathcal{N}}(\mu=0.6, \sigma=0.05)$ and
  $X_{\text{right}}\sim0.5\bar{\mathcal{N}}(\mu=0.3, \sigma=0.03) +
  0.5\bar{\mathcal{N}}(\mu=0.5, \sigma=0.15)$. For this selection of value
  distributions, we run EQR to estimate $m=10$ quantiles. 
\end{example}

The middle row of \cref{fig:qr_examples} shows that, for $\beta=0$, the prediction error oscillates
close to zero for every quantile, thus validating the result of \cref{corollary:convergence}. To
test the prediction quality when the assumptions are violated, we generate different values for
$\beta \in [0, 1]$ and run EQR for the same three MDP posterior distributions. The bottom plots in
\cref{fig:qr_examples} show the 1-Wasserstein metric between the true and predicted quantile
distributions after $10^4$ gradient steps; the error, like the covariance between $V(s_0)$ and
$P(s_2 | s_0)$, increases monotonically with $\beta$. The prediction quality thus degrades depending
on the magnitude of the covariance between the transition kernel and the values.

\cref{ex:toy_mdp} is mostly pedagogical and serves the purpose of validating our theoretical result,
but it remains a contrived example with limited scope. The next example analyzes the performance of
EQR in a standard tabular problem.

\begin{example}[Gridworld]
  We consider a modification of the $N$-room gridworld environment by
  \citet{domingues_rlberry_2021}, consisting of three connected rooms of size $5 \times 5$. The task
  for this example is to predict $m=100$ quantiles of the value distribution under the optimal
  policy $\pi^\star$ (obtained by running any standard tabular exploration algorithm, like PSRL
  \citep{osband_more_2013}). We use a Dirichlet prior for the transition kernel and a standard
  Gaussian for the rewards. We collect data using $\pi^\star$, update the posterior MDP and run EQR
  to predict the value distribution. 
\end{example}

In \cref{fig:nroom_quantile_estimation}, we summarize the results at three
different points during data collection. As more data is collected, the
corresponding MDP posterior shrinks and we observe the value distribution
concentrates around the value under the true dynamics (dotted vertical line).
For both wide (episode 1) and narrow (episode 100) posteriors, EQR is able to
accurately predict the distribution of values. The impact of violating
\cref{assumption:acyclic} is the non-zero steady-state quantile-regression
error. We observe the bias of the predicted quantiles is typically lowest (near
zero) close to the median and highest at the extrema.

\begin{figure}[t]
  \includegraphics{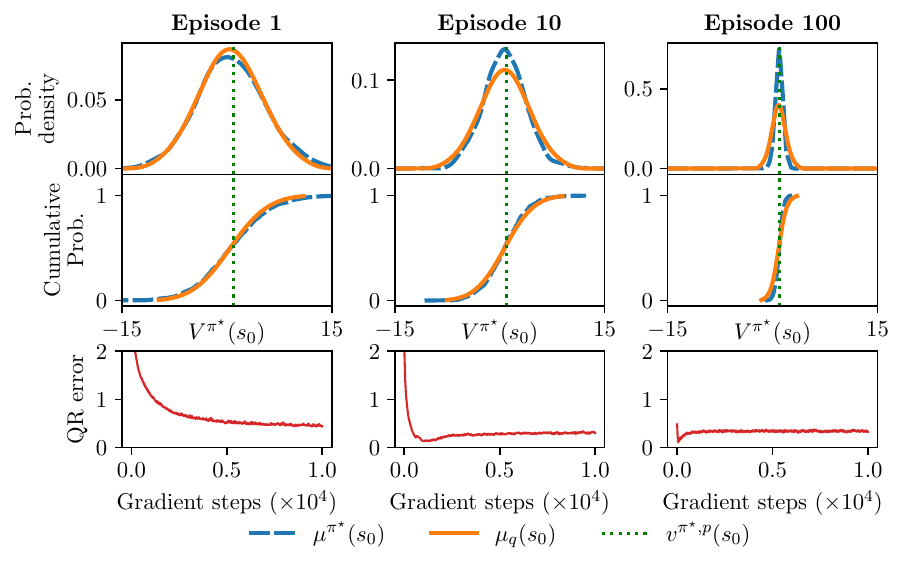}
  \caption{Performance of EQR in a Gridworld environment. We train the optimal policy $\pi^\star$
  using PSRL \citep{osband_more_2013} and then use it for data collection. At different points
  during data collection, we run EQR to estimate $m=100$ quantiles of the value distribution for the
  initial state under $\pi^\star$, given the current posterior MDP. \textbf{(Top-Middle)} PDF and
  CDF of the true (dashed blue) and predicted (solid orange) value distribution, with the true
  optimal value (dotted green) as vertical reference. \textbf{(Bottom)} $1$-Wasserstein distance
  between the true quantile projection and the prediction.}
  \label{fig:nroom_quantile_estimation}
\end{figure}

\section{Policy Optimization with Value Distributions}
In this section, we propose an algorithm to optimize a policy given some
differentiable utility function $f$ of the learned quantile distribution $\mu_q$ (which is
implicitly parameterized by $\pi$). Namely, we define the optimal policy
\begin{equation}
  \label{eq:objective}
  \pi^\star = \argmax_\pi f(\mu_q;\pi).
\end{equation}
To approximately solve \cref{eq:objective}, we combine EQR with SAC (EQR-SAC) to obtain a
model-based reinforcement learning algorithm that leverages a value-distributional critic for policy
optimization. Our algorithm is agnostic to $f$ as long as it is differentiable and thus can
backpropagate gradients through it. The key ingredients of our method are: (1) an ensemble-based
posterior over MDPs, (2) a quantile-distributional critic network that models the $m$-quantile
function $q(s,a)$ and (3) a policy network $\pi_\phi$ trained to optimize \cref{eq:objective}.
A full algorithmic description of EQR-SAC is included in \cref{app:implementation_details}.

\paragraph{Posterior Dynamics.} We adopt the baseline architecture from MBPO
\citep{janner_when_2019} and the implementation from \citet{pineda_mbrl-lib_2021}, where the
posterior MDP, denoted $\Gamma_\psi$, is represented as an ensemble of $n$ neural networks trained
via supervised learning on the environment dataset $\mathcal{D}$ to predict the mean and variance of
a Gaussian distribution over next states and rewards. We use $\Gamma_\psi$ to populate an experience
replay buffer $\mathcal{D}_{\text{model}}$ with model-consistent $k$-step rollouts; that is, we use
a consistent ensemble member throughout a rollout, rather than randomizing the model per-step like
in MBPO. 

\paragraph{Critic.} We train the critic on mini-batches drawn from $\mathcal{D}_{\text{model}}$,
use the entropy regularization loss from SAC with temperature $\alpha$ and replace the quantile
regression loss with the quantile Huber loss $\rho_\tau(u)$ from \citet{dabney_distributional_2018}
(see \cref{subsec:quantile_huber_loss})
\begin{equation}
  \label{eq:critic_loss}
  \mathcal{L}_{\text{critic}}(q) = \E_{(S,A) \sim \mathcal{D_{\text{model}}}}\bracket{
    \E_{(\hat{R}, \hat{P}) \sim \Gamma_\psi} \bracket{
      \sum_{i=1}^m \E_J \bracket{
        \rho_{\tau_i}\paren{\mathcal{T}q_J(S,A) - q_i(S, A)}
      }
    }
  },
\end{equation}
where the target quantiles $\mathcal{T}q_j$ are defined as
\begin{equation}
  \label{eq:target_quantile}
  \mathcal{T}q_j(s, a) = \hat{R}(s, a) + \gamma \E_{(S', A') \sim \hat{P}(\cdot \mid s,a), \pi_\phi} \bracket{q_j(S',A') - \alpha \log \pi_\phi(A' \mid S')}.
\end{equation}
The expectation in \cref{eq:target_quantile} is approximated by generating transition tuples $(s',
a')$ using the policy and the sampled dynamics from $\Gamma_\psi$. Typically, model-based algorithms
like MBPO only use data in the mini-batch to compose critic targets, rather than leveraging the
learned dynamics model for better approximation of expectations.

\paragraph{Actor.} The policy is trained to maximize the objective in \cref{eq:objective}, in addition to the entropy regularization term from SAC,
\begin{equation}
  \label{eq:actor_loss}
  \mathcal{L}_{\text{actor}}(\phi) = \E_{S \sim {\mathcal{D}_{\text{model}}}} \bracket{
    \E_{A \sim \pi_\phi}\bracket{ 
      f(q(S, A)) - \alpha \log \pi_\phi(A \mid S)
    }
  }.
\end{equation}
Let $\bar{q}(s,a)$ and $\sigma_q(s,a)$ be the mean and standard deviations of the
quantiles, respectively. Then, we consider two concrete utility functions: the classical mean
objective $f_{\text{mean}}(q(s,a)) = \bar{q}(s,a)$ and an objective based on optimism in
the face of uncertainty $f_{\text{ofu}} = \bar{q}(s,a) + \sigma_q(s,a)$.

\section{Experiments}
\label{sec:experiments}
In this section, we evaluate EQR-SAC in environments with continuous state-action spaces.
Implementation details and hyperparameters are included in
\cref{app:implementation_details,app:hparams}, respectively. Unless noted otherwise, all training
curves are smoothened by a moving average filter and we report the mean and standard error over 10
random seeds.

\subsection{Baselines}
We consider the following baselines, which all share a common codebase and hyperparameters.

\textbf{SAC} with typical design choices like target networks \citep{mnih_playing_2013}, clipped
double Q-learning \citep{fujimoto_addressing_2018} and automatic entropy tuning
\citep{haarnoja_soft_2019}.

\textbf{MBPO} with slight modifications from \citet{janner_when_2019}: (1) it only uses
$\mathcal{D}_{\text{model}}$ to update the actor and critic, rather than mixing in data from
$\mathcal{D}$; (2) it uses a fixed rollout length $k$, instead of an adaptive scheme. With respect
to EQR-SAC, MBPO collects data differently: instead of collecting $k$-step rollouts under each model
of the ensemble, it does so by uniformly sampling a new model per step of the rollout. 

\textbf{QR-MBPO}, which replaces the critic in MBPO with a quantile-distributional critic,
trained on the standard quantile-regression loss from \citet{wurman_outracing_2022}, but using data
from $\mathcal{D}_{\text{model}}$,
\begin{equation}
  \label{eq:critic_loss_qrmbpo}
  \mathcal{L}_{\text{critic}}^{\text{qrmbpo}}(q) = \E_{(S,A,S',R) \sim \mathcal{D_{\text{model}}}}\bracket{
    \bracket{
      \sum_{i=1}^M \E_J \bracket{
        \rho_{\tau_i}\paren{\mathcal{T}q^{\text{qrmbpo}}_J(R,S,A) - q_i(S, A)}
      }
    }
  },
\end{equation}
where the target quantile is defined as
\begin{equation}
  \mathcal{T}q_j^{\text{qrmbpo}}(r, s', a) = r + \gamma\paren{q_j(s', a') - \alpha \log \pi_\phi(a' \mid s')},
\end{equation}
and $a' \sim \phi_\phi(a' \mid s')$. Importantly, \cref{eq:critic_loss_qrmbpo} differs from
\cref{eq:critic_loss} in that the former captures both the aleatoric and epistemic uncertainty
present in $\mathcal{D}_{\text{model}}$, while the latter aims to average out the aleatoric noise
from the target quantiles. The objective function for the actor is the same as EQR-SAC.

\paragraph{QU-SAC,} as proposed by \citet{luis_model-based_2023}. It collects data as in EQR-SAC,
but stores the $n$ model-consistent rollouts in $n$ separate buffers (while EQR-SAC uses a single
buffer). Then it trains an ensemble of $n$ standard critics on the corresponding $n$ model-buffers.
As such, it interprets the ensemble of critics as samples from the value distribution. The actor is
optimized to maximize the mean prediction of the critic ensemble.

\subsection{Case Study - Mountain Car}
\label{subsec:mountaincar}
We motivate the importance of learning a distribution of values with a simple environment, the
Mountain Car \citep{sutton_reinforcement_2018} with continuous action space, as implemented in the
\texttt{gym} library \citep{brockman_openai_2016}. The environment's rewards are composed of a small
action penalty and a large sparse reward once the car goes up the mountain, defined by a horizontal
position $x > 0.45$ meters. We consider three versions of the problem: the original one, and two
variants where all the rewards are scaled by a constant factor of 0.5 and 0.1, respectively.

\begin{figure}[t]
  \includegraphics{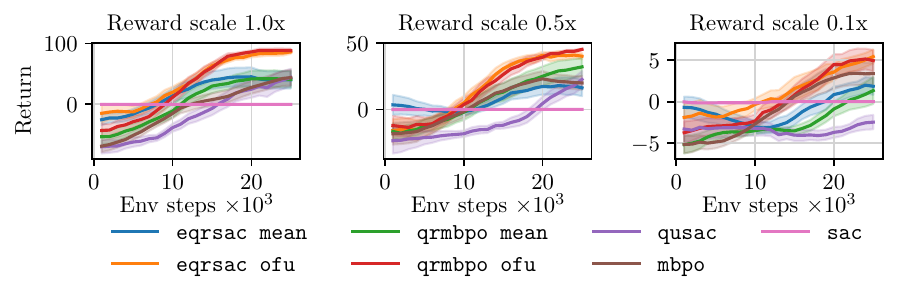}
  \caption{Performance in the Mountain Car environment. We consider the original version of the
  environment (left) and two variants (middle, right) that scale down the rewards by some factor
  (0.5 and 0.1, respectively).}
  \label{fig:main_mountaincar}
\end{figure}

\begin{figure}[t]
  \includegraphics{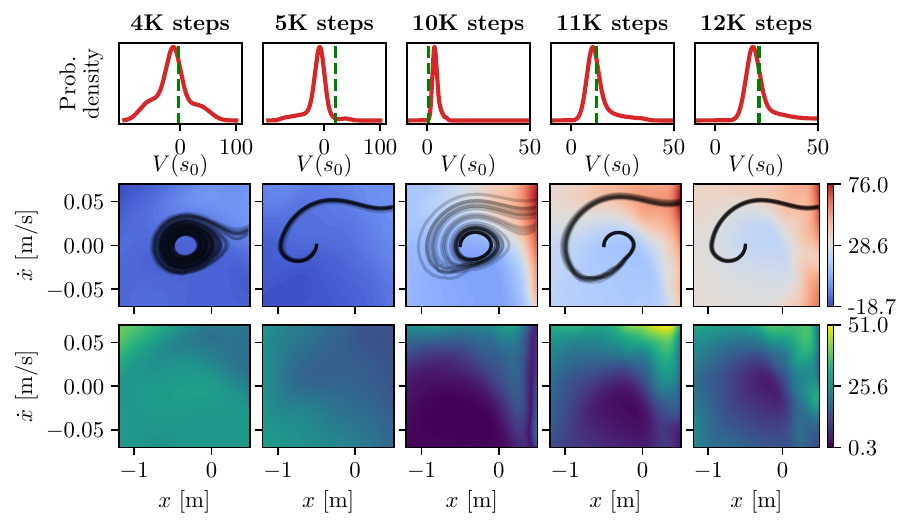}
  \caption{Visualization of the learned value distribution of EQR-SAC at different points during
  training in the Mountain Car environment (with reward scale of 0.5x). \textbf{(Top)} The predicted
  value distribution at the initial state. The dotted line is the empirical value of $s_0$ based on
  ten trajectories (black lines of middle row). \textbf{(Mid-Bottom)} The mean (mid) and standard
  deviation (bottom) of the value distribution across the state space.}
  \label{fig:viz_mountaincar}
\end{figure}

While it is low-dimensional and has simple dynamics, many RL algorithms fail to solve
the Mountain Car problem due to its combination of action penalties and sparse rewards.
Naive exploration strategies based on injecting unstructured noise, like SAC, typically
fail to solve such tasks \citep{raffin_smooth_2021}. We plot the performance of EQR-SAC
and all baselines in \cref{fig:main_mountaincar}. EQR-SAC and QR-MBPO have the best
overall performance, both using the optimistic objective function $f_{\text{ofu}}$
(as previously defined after \cref{eq:actor_loss}). These results highlight the
need to model uncertainty \emph{and} leverage it during optimization; optimizing the
mean values significantly degraded performance of the distributional approaches.

In \cref{fig:viz_mountaincar}, we inspect further the distribution of values learned by
EQR-SAC during a training run. The value distribution is initially wide and
heavy-tailed, as the agent rarely visits goal states. At 5K steps, the policy is
close-to-optimal but the predicted distribution underestimates the true values. In
subsequent steps, the algorithm explores other policies while reducing uncertainty and
calibrating the predicted value distribution. At 12K steps, the algorithm stabilizes
again at the optimized policy, but with a calibrated value distribution whose mean is
close to the empirical value. We notice the large uncertainty in the top-right corner of
the state space remains (and typically does not vanish if we run the algorithm for
longer); we hypothesize this is mainly an artifact of the discontinuous reward function,
which is smoothened out differently by each member of the ensemble of dynamics, such
that epistemic uncertainty stays high.

\begin{figure}[t]
  \includegraphics{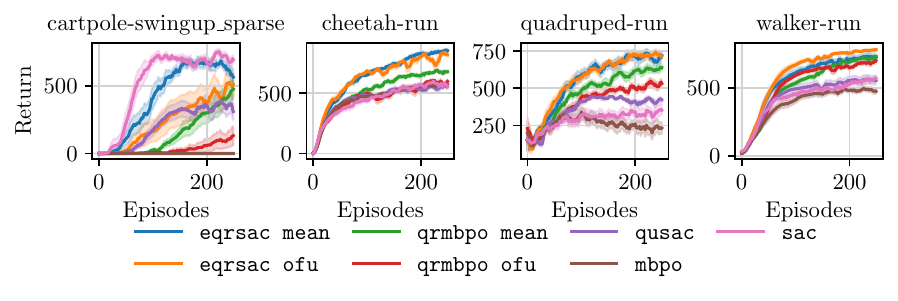}
  \caption{Performance in four DeepMind Control tasks. Cartpole swing-up has sparse rewards,
  while Cheetah, Quadruped and Walker have dense rewards. EQR-SAC significantly improves
  performance with respect to the model-based baselines.}
  \label{fig:main_dm_control}
\end{figure}

\subsection{DM Control Benchmark}
In order to evaluate EQR-SAC more broadly, we conduct an experiment in a subset of 16
continuous-control tasks from the DeepMind Control Suite \citep{tunyasuvunakool_dm_control_2020}.
The chosen environments include both dense/sparse rewards and smooth/discontinuous dynamics. In
\cref{fig:main_dm_control}, we plot the results for four environments ranging from small (cartpole)
to mid/large (quadruped) observation spaces. Our method significantly improves performance over
previous model-based algorithms in these environments. Moreover, in the full benchmark, EQR-SAC
achieves the best (or comparable) final performance in 13 out of 16 tasks (see
\cref{app:dm_control_scores}). We summarize the results of the DMC benchmark in
\cref{fig:agg_perf}, following the guidelines by \citet{agarwal_deep_2021}. At 250 episodes of
training, EQR-SAC OFU achieves the highest normalized IQM score, which is $\sim 17\%$ higher than
QR-MBPO mean (see \cref{tab:agg_dm_control_scores} for numerical scores). However, there exists some
overlap between the 95\% confidence intervals, which tend to be large in our benchmark due to a wide
range of normalized scores across different environments. In this scenario, the recommendation in
\citet{agarwal_deep_2021} is to analyze score distribution performance profiles, as presented in
\cref{fig:agg_perf}, which provide a more complete overview of the results. We observe the EQR-SAC
OFU performance profile tends to dominate over the baselines, especially for normalized scores
between $[0.5, 0.8]$.

We observe a clear gap in performance between MBPO and QR-MBPO, which supports the observations from
\citet{lyle_comparative_2019} and reinforces their hypothesis that distributional critics boost
performance under non-linear function approximation. The gap between QU-SAC and the distributional
methods (QR-MBPO / EQR-SAC) indicates that the quantile representation of values leads to more
sample-efficient learning compared to the ensemble-based approach. Moreover, training one
distributional critic is typically less computationally intensive than training an ensemble of
standard critics. In the next section, we investigate more deeply the performance difference between
EQR-SAC and QR-MBPO.

\begin{figure}[t]
  \includegraphics[width=1.0\textwidth]{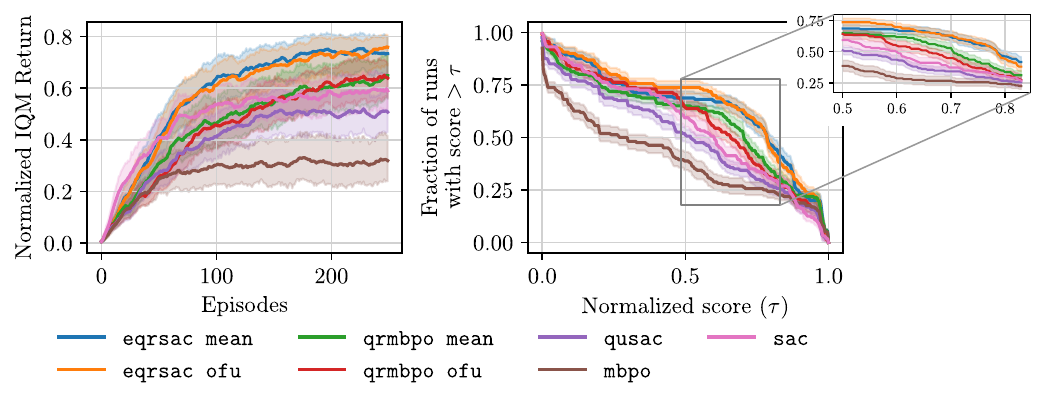}
  \caption{Aggregated performance in DMC benchmark with 95\% bootstrap confidence intervals.
  \textbf{(Left)} Inter-quartile mean returns normalized by the maximum achievable score of 1000.
  \textbf{(Right)} Performance profile at 250 episodes of training, with zoom in region with the
  most spread in results. In both cases, higher curves correspond to better performance.}
  \label{fig:agg_perf}
\end{figure}

\subsection{Why Does EQR-SAC Outperform QR-MBPO?}
\label{subsec:qrmbpo_investigation}
We conduct an additional experiment to determine what component(s) of EQR-SAC are responsible for
the increased performance with respect to QR-MBPO. There is three differences between EQR-SAC and
QR-MBPO: (i) EQR-SAC has a buffer $n$ times bigger than QR-MBPO, and correspondingly scales up the
amount of data collected under the ensemble, (ii) EQR-SAC uses consistent rollouts, while QR-MBPO
randomizes the model per-step, and (iii), EQR-SAC's critic is trained on the loss
\cref{eq:critic_loss}, while QR-MBPO's critic uses \cref{eq:critic_loss_qrmbpo}. In order to test
the impact of each component in isolation, we add three additional baselines: QR-MBPO-big, which
uses the same buffer size and collects the same amount of data as EQR-SAC; QR-MBPO-consistent, that
replicates how EQR-SAC collects data under the model; and QR-MBPO-loss, that uses
\cref{eq:critic_loss} to train its critic. \cref{fig:qrmbpo_investigation} shows the performance of
EQR-SAC and all QR-MBPO variants (all methods optimize the actor using $f_{\text{mean}}$). The main
observation is that QR-MBPO-loss matches closely the performance of EQR-SAC, while all other QR-MBPO
variants share similar (lower) performance. The key insight from these results is that our proposed
critic loss function \cref{eq:critic_loss} is instrumental towards sample-efficient learning,
especially in environments with sparse rewards like cartpole swing-up (see also fish-swim and
finger-spin in \cref{app:dm_control_curves}). As such, our theory provides a solid guideline on how
to integrate model-based RL architectures with distributional RL tools, which goes beyond simply
using a distributional critic with established algorithms like MBPO.

\begin{figure}[t]
  \includegraphics{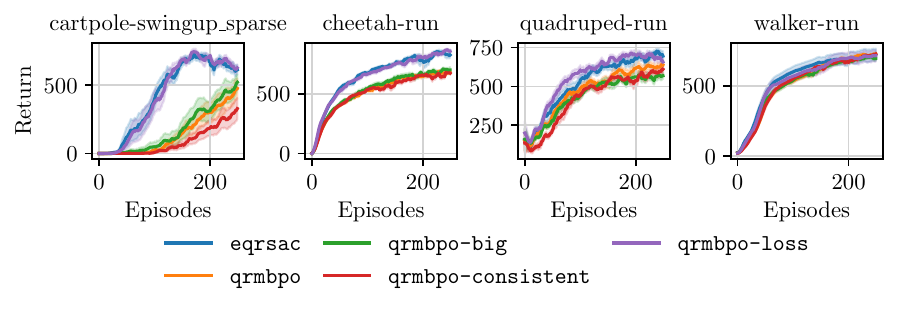}
  \caption{Comparison of EQR-SAC and QR-MBPO baselines in selected DeepMind Control tasks. The results
  suggest that the biggest contributing factor for increased performance of EQR-SAC w.r.t QR-MBPO is
  the critic's loss function \cref{eq:critic_loss}.}
  \label{fig:qrmbpo_investigation}
\end{figure}

\begin{figure}[t]
  \includegraphics{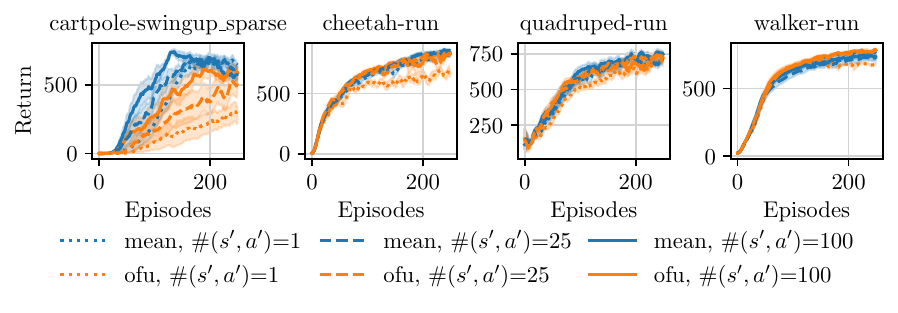}
  \caption{Ablation study on the amount of next state-action samples drawn to approximate the target
  quantiles \cref{eq:target_quantile}. Larger sample sizes perform more robustly across all
  environments.}
  \label{fig:ablation_num_samples}
\end{figure}

\subsection{Dynamics Sampling for Target Quantiles}
The analysis in \cref{subsec:qrmbpo_investigation} points to the loss function \cref{eq:critic_loss}
as being the key component of our proposed approach. The main feature of our loss function is how it
utilizes the generative model to produce the target quantiles \cref{eq:target_quantile}. In this
experiment, we investigate the effect of the amount of next state-action samples $(s', a')$ drawn
from the ensemble of dynamics when estimating the target quantiles. The hypothesis is that a larger
sample size would result in a lower-variance estimate of the expectation in
\cref{eq:target_quantile}, which could then lead to better sample-efficiency.
\cref{fig:ablation_num_samples} shows the performance of EQR-SAC, for both $f_{\text{mean}}$ and
$f_{\text{ofu}}$, under different sampling regimes. For the cartpole task, we observe a clear
progression in sample-efficiency as the amount of samples increases. For other environments the
differences are less noticeable, but using a sample size of $1$ with the optimistic objective leads
to worse performance in all cases. Since the sample size might have large effects in performance and
its runtime impact is greatly amortized by GPU parallelization, the overall recommendation is to use
larger sample sizes (25-100) by default.

\subsection{Additional Experiments}
In \cref{app:ablations}, we include additional ablation studies on three hyperparameters of our
algorithm: the number of quantiles $(m)$, the rollout length $(k)$ and the size of
$\mathcal{D}_{\text{model}}$ (which controls the amount of off-policy data in the buffer). The
general observation from these experiments is that EQR-SAC's performance is robust for a wide range
of values. Performance typically degrades only for extreme values of the parameters, for example
$m=1$ (only estimate median of value distribution) or $k=1$ (only generate 1-step rollouts with the
model).

Furthermore, in \cref{app:optimistic_opt} we conduct experiments to investigate the utility of
optimistic values in policy optimization. The empirical results suggest that optimism has little
effect under dense rewards, while under sparse rewards higher optimism can be benefitial in some
tasks and detrimental in others. We also study the effect of optimism in tasks with sparse rewards
\emph{and} action costs, similar to the MountainCar problem of \cref{subsec:mountaincar}. For the
tested environments, higher optimism does not improve final performance and is typically less sample
efficient, unlike the MountainCar results. Overall, all these results indicate that the benefits of
optimistic optimization might be environment-dependent. We believe an interesting avenue for future
work is to more broadly analyze this phenomenon and reconsider our design choices (ensemble as
posterior MDP, quantile representation for values, policy objectives, etc.).

\section{Conclusions}
We investigated the problem of estimating the distribution of values, given parameter uncertainty of
the MDP. We proposed the value-distributional Bellman equation and extracted an operator whose
fixed-point is precisely the distribution of values. Leveraging tools from return-distributional RL,
we designed Epistemic Quantile-Regression, an iterative procedure for estimating quantiles of the
value distribution. We applied our algorithm in small MDPs, validated the convergence properties
prescribed by our theory and assessed its limitations once the main assumptions are violated.
Lastly, we introduced EQR-SAC, a novel model-based deep RL algorithm that scales up EQR with
neural network function approximation and combines it with SAC for policy optimization. We
benchmarked our approach in several continuous-control tasks from the DeepMind Control suite and
showed improved sample-efficiency and final performance compared to various baselines.

\clearpage
\appendix
\section{Unrolling MDP with Cycles}
\label{app:unrolling}
In \cref{fig:unrolling} we show the unrolling procedure of an MDP that contains cycles.
The two MDPs are \emph{not} equivalent, but the unrolled approximation improves as $H$
grows. 

While unrolling the MDP is rarely done in practice, it serves as a reasonable approximation to extend our theoretical results to the general setting of MDPs that contain cycles.

\begin{figure}[H]
  \includegraphics[width=\textwidth]{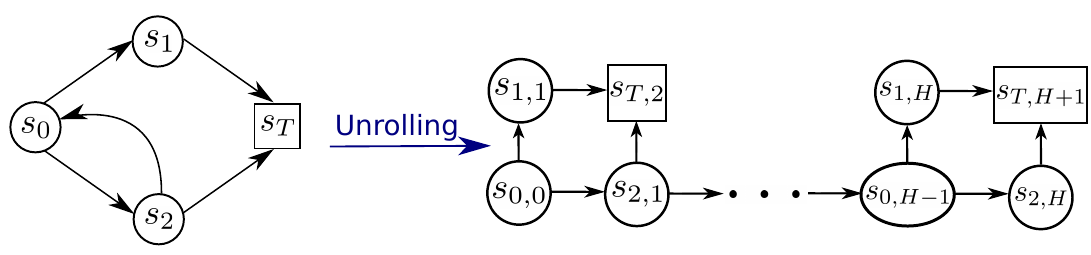}
  \caption{Procedure of unrolling an MDP with cycles. We denote by $s_{i, k}$ the
  unrolled state which represents being in state $s_i$ of the original MDP at time step
  $k$. The unrolled MDP is only an \emph{approximation} of the original version due to
  truncation after finite $H$ steps.}
  \label{fig:unrolling}
\end{figure}

\section{Theory Proofs}
\label{app:theory}
\randomvarbellman*
\begin{proof}
  We proceed similarly as the standard Bellman equation proof shown by
  \citet{bellemare_distributional_2023}. First, the random trajectories $\tilde{T}$ have two
  properties endowed by the Markov decision process: time homogeneity and the Markov property.
  Informally speaking, time homogeneity states that the trajectory from a given state $s$ is
  independent of the time $k$ at which the state is visited, while the Markov property states that
  trajectories starting from $s$ are independent of states, actions or rewards encountered before
  $s$ (c.f. \citet{bellemare_distributional_2023} Lemmas 2.13, 2.14 for a formal definition). In the
  domain of random variables, these properties imply that two trajectories starting from the same
  initial state $s$ are equally distributed regardless of past history.

  From the definition \cref{eq:random_value} we decompose the random value into the immediate reward
  and the value at the next state:
  \begin{align}
    V^\pi(s) &= \E_{\tilde{T}}\bracket{
      R_0 \middle| S_0 = s,\, P
    } + \gamma\E_{\tilde{T}} \bracket{
      \sum_{h=1}^\infty \gamma^{h-1} R_h \middle| S_0=s,\, P
    }. \intertext{
      For the first term, the only random variable remaining is $A_0$, so we rewrite it as
    }
    &= \sum_a \pi(a\mid s)r(s, a) + \gamma\E_{\tilde{T}} \bracket{
      \sum_{h=1}^\infty \gamma^{h-1} R_h \middle| S_0=s,\, P}. \intertext{
        For the second term, we apply the tower property of expectations
      }
    &= \sum_a \pi(a\mid s)r(s, a) + \gamma\E_{\tilde{T}} \bracket{ \E_{\tilde{T}} \bracket{
      \sum_{h=1}^\infty \gamma^{h-1} R_h \middle| S_0=s,\,A_0,\,S_1,\, P} \middle| S_0=s,\,P}. \intertext{
        By the Markov property, 
      }
    &= \sum_a \pi(a\mid s)r(s, a) + \gamma\E_{\tilde{T}} \bracket{ \E_{\tilde{T}} \bracket{
      \sum_{h=1}^\infty \gamma^{h-1} R_h \middle| S_1,\, P} \middle| S_0=s,\,P}. \intertext{
        By time homogeneity, the inner expectation is exactly equal to the random variable $V^\pi(S_1)$, after a change of variable in the infinite sum index
      }
    &= \sum_a \pi(a\mid s)r(s, a) + \gamma\E_{\tilde{T}}\bracket{V^\pi(S_1) \middle| S_0=s,\,P}. \intertext{
      Lastly, the remaining random variable is $S_1$, for which we can
      explicitly write its probability distribution, concluding the proof
      }
    &= \sum_a \pi(a\mid s)r(s, a) + \gamma\sum_{a, s'} \pi(a\mid s)P(s' \mid s,a)V^\pi(s').
  \end{align}
\end{proof}

\textbf{Notation:} for the following Lemma, we use the notation $\mathfrak{D}(X)$ to
denote the distribution of the random variable $X \in \mathcal{X}$, as done in
\citet{bellemare_distributional_2023}. In particular, we use $\mathfrak{D}_P$ to denote
the distribution of a random variable belonging to the probability space of $P$, i.e.,
random variables derived from the posterior distribution $\Phi(P \mid \mathcal{D})$.

\bellmandist*
\begin{proof}
  In matrix-vector format, the random-variable value-distributional Bellman equation is expressed as
  \begin{equation}
    \label{eq:lemma_1_first}
    \mathbf{V}^\pi = \mathbf{r}^\pi + \gamma \mathbf{P}^\pi \mathbf{V}^\pi.
  \end{equation}
  Since $\mu^\pi$ refers to the distribution of the random variable
  $\mathbf{V}^\pi$, which belongs to the probability space of the random transition
  function $P$, then we use our notation to write $\mu^\pi =
  \mathfrak{D}_P(\mathbf{V}^\pi)$. Further, using the notation $\mathfrak{D}_{P}(\cdot)$
  on the r.h.s of \cref{eq:lemma_1_first} yields
  \begin{equation}
    \label{eq:lemma_1_second}
    \mu^\pi = \mathfrak{D}_{P}(\mathbf{r}^\pi + \gamma \mathbf{P}^\pi \mathbf{V}^\pi).
  \end{equation}
  For any two random variables $X, Y$ in the same probability space, it holds that
  the marginal distribution over $X$ can be written as the expected value over $Y$ of
  the conditional distribution. That is, $\mathfrak{D}(X) = \E_Y\bracket{\mathfrak{D}(X
  \mid Y)}$, which follows from standard probability theory \citep{wasserman_all_2013}.
  Applying this property to the r.h.s of \cref{eq:lemma_1_second} results in
  \begin{equation}
    \mu^\pi = \E_P\bracket{\mathfrak{D}_{P}(\mathbf{r}^\pi + \gamma \mathbf{P}^\pi \mathbf{V}^\pi \mid \mathbf{P}^\pi)}.
  \end{equation}
  A similar derivation can be found in related prior work studying the variance of
  $\mu^\pi$
  \citep{odonoghue_uncertainty_2018,zhou_efficient_2019,luis_model-based_2023}.

  Given that $P(s' \mid s,a)$ and $V^\pi(s')$ are independent under our assumptions, then
  conditioning on $\mathbf{P}^\pi$ means that the distribution of the  matrix-vector product
  $\mathbf{P}^\pi \mathbf{V}^\pi$ is simply the distribution of applying a linear transformation on
  $\mathbf{V}^\pi$. The result is that the conditional distribution can be interpreted as the
  pushforward
  \begin{equation}
    \mathfrak{D}_{P}(\mathbf{r}^\pi + \gamma \mathbf{P}^\pi \mathbf{V}^\pi \mid \mathbf{P}^\pi) = (b_{r^\pi, P^\pi, \gamma})_{\#}\mu^\pi,
  \end{equation}
  which completes the proof.
\end{proof}

We adopt the supremum $p$-Wasserstein distance to establish contractivity of the
operator~$\mathcal{T}^\pi$.
\begin{definition}
  \label{def:wasserstein}
  For $p \in [1, \infty)$, the $p$-Wasserstein distance between two distributions $\nu, \nu'$ is a
  metric $w_p: \mathcal{P}(\R) \times \mathcal{P}(\R) \to [0, \infty]$ defined by
  \begin{equation}
    w_p(\nu, \nu') = \paren{\int_0^1 \abs{F^{-1}_\nu(\tau) - F^{-1}_{\nu'}(\tau)}^p d\tau}^{1 / p},
  \end{equation}
  where $F^{-1}_{(\cdot)}$ is the inverse cumulative distribution function. Furthermore, the
  supremum $p$-Wasserstein distance $\bar{w}_p$ between two value distribution functions $\mu,\mu'
  \in \mathcal{P}(\R)^{\mathcal{S}}$ is defined by $\bar{w}_p(\mu, \mu') = \sup_{s \in \mathcal{S}}
  w_p(\mu(s), \mu'(s))$.
\end{definition}
The supremum $p$-Wasserstein distance was proven to be a metric in $\mathcal{P}(\R)^{\mathcal{S}}$
by \citet{bellemare_distributional_2017}. 

To prove that $\mathcal{T}^\pi$ is a contraction, we adopt the technique from
\citet{bellemare_distributional_2023} that relies on the alternative definition of the
$p$-Wasserstein distance in terms of couplings.

\begin{definition}[Coupling \citep{villani_optimal_2008} Definition 1.1 (adapted)]
  \label{def:wasserstein_coupling}
  Let $\nu, \nu' \in \mathcal{P}(\R)$ be two probability distributions over the reals. A
  coupling between $\nu$ and $\nu'$ is a joint probability distribution $\upsilon \in
  \mathcal{P}(\R^2)$ whose marginals are $\nu$ and $\nu'$. That is, given random
  variables $(V, V') \sim \upsilon$, we have $V \sim \nu$ and $V' \sim \nu'$. Further,
  we denote $\Gamma(\nu, \nu') \subseteq \mathcal{P}(\R^2)$ the set\footnote{This
  set is non-empty: there exists a trivial coupling in which the variables $V$, $V'$ are
  independent \citep{villani_optimal_2008}} of all couplings between $\nu$ and $\nu'$.
\end{definition}

Intuitively, the coupling $\upsilon$ can be interpreted as a transport plan to move
probability mass from one distribution to another. The $p$-Wasserstein distance can also
be defined as the cost of the optimal transport plan
\begin{equation}
  \label{eq:wasserstein_min}
  w_p(\nu, \nu') = \min_{\upsilon \in \Gamma(\nu, \nu')} \E_{(V, V') \sim \upsilon} \bracket{
    \abs{V - V}^p
  }^{1/p}.
\end{equation}
The existence of an optimal coupling $\upsilon^\star$ that minimizes
\cref{eq:wasserstein_min} is guaranteed since $\nu, \nu'$ are measures defined on a
complete, separable metric space ($\R$ with the usual metric) and equipped with the
corresponding Borel $\sigma$-algebra (i.e., $\nu, \nu'$ are measures on a Polish space)
\citep[cf.][Theorem 4.1]{villani_optimal_2008}.

With these definitions, we now proceed to prove the contraction of the Bellman operator.
\contraction*
\begin{proof}
  We follow closely the proofs of Proposition 4.1 by
  \citet{amortila_distributional_2020} and Proposition 4.15 by
  \citet{bellemare_distributional_2023}. For each $s \in \mathcal{S}$, let
  $\upsilon^\star$ denote the optimal coupling that minimizes the $p$-Wasserstein metric
  from \cref{def:wasserstein} between some arbitrary pair of value distributions
  $\mu(s), \mu'(s) \in \mathcal{P}(\R)$, so that $(V(s), V'(s)) \sim \upsilon^\star$.

  Define new random variables $\tilde{V}(s) = r^\pi(s) + \gamma \sum_{s'}P^\pi(s'\mid s) V(s')$,
  $\tilde{V}'(s) = r^\pi(s) + \gamma \sum_{s'}P^\pi(s'\mid s) V'(s')$. By definition of the operator
  $\mathcal{T}^\pi$, we have that $\tilde{V}(s) \sim (\mathcal{T}^\pi \mu)(s)$ and $\tilde{V}'(s)
  \sim (\mathcal{T}^\pi \mu')(s)$, which means that the pair $(\tilde{V}(s), \tilde{V}'(s)) \sim
  \tilde{\upsilon}$ is a coupling between $(\mathcal{T}^\pi \mu)(s)$ and $(\mathcal{T}^\pi
  \mu')(s)$.

  Starting from \cref{def:wasserstein_coupling} and since the $p$-Wasserstein distance
  is a minimum over couplings, then
  \begin{align}
    w_p^p \paren{(\mathcal{T}\mu)(s), (\mathcal{T}\mu')(s)} &\leq \E_P\bracket{\abs{\tilde{V}(s) - \tilde{V}'(s)}^p}.
    \intertext{Plugging the definition of the random variables,}
    &=\E_P\bracket{\abs{
      r^\pi(s) + \gamma\sum_{s'}P^\pi(s'\mid s)V(s') - r^\pi(s) - \gamma \sum_{s'}P^\pi(s' \mid s)V'(s')}^p}
    \intertext{By re-arrangement of terms}
    &=\gamma^p\E_P\bracket{\abs{
      \sum_{s'}P^\pi(s'\mid s)(V(s') - V'(s'))}^p}.
    \intertext{Since $f(x) = \abs{x}^p$ is convex for $p \geq 1$, then by Jensen's inequality}
    &\leq\gamma^p\E_P\bracket{
      \sum_{s'}P^\pi(s'\mid s)\abs{(V(s') - V'(s'))}^p}.
    \intertext{By linearity of expectation}
    &=\gamma^p\sum_{s'}\E_P\bracket{P^\pi(s'\mid s)\abs{(V(s') - V'(s'))}^p}.
    \intertext{By the independence assumption on $P^\pi(s' \mid s)$, the expectation of the
    product becomes the product of expectations}
    &=\gamma^p\sum_{s'}\E_P\bracket{P^\pi(s'\mid s)}\E_P\bracket{\abs{(V(s') - V'(s'))}^p}.
    \intertext{Since the supremum of non-negative values is greater or equal
    than any convex combination of them}
    &\leq\gamma^p \sup_{s'}\E_P\bracket{\abs{(V(s') - V'(s'))}^p}.
    \intertext{By definition of the supremum $p$-Wasserstein distance}
    &=\gamma^p \bar{w}_p^p(\mu, \mu').
  \end{align}
  Taking supremum on the left-hand side and taking the $p$-th root on both sides completes the
  proof.
\end{proof}
\Cref{thm:contraction} parallels similar results in standard RL and model-free distributional RL, in
that it allows us to establish the convergence of iterated applications of $\mathcal{T}^\pi$
(\cref{corollary:convergence}).
\subsection{Supporting Lemmas}
\begin{lemma}
  \label{lemma:independence}
  Under \cref{assumption:transitions,assumption:acyclic,assumption:truncated},
  $P(s' \mid s, a)$ and $V^\pi(s')$ are conditionally independent random variables
  for all triplets $(s, a, s') \in \mathcal{S}\times \mathcal{A}\times \mathcal{S}$.
\end{lemma}
\begin{proof}
  Let $T_{0:\infty}$ be a random trajectory under the random transition dynamics $P$.
  Under \cref{assumption:acyclic,assumption:truncated}, $T_{0:\infty}$ is a
  sequence of $H$ random, but \emph{unique} states followed by the terminal (absorbing)
  state $\set{S_0, S_1, \dots, S_{H}, s_T, s_T, \dots}$, i.e., we have $S_i \neq S_j$
  for all $i \neq j$. We prove the Lemma by dividing the random trajectory into two
  segments: the random (finite) segment for step $h \leq H$ and the deterministic
  (infinite) segment for $h > H$.
  
  \paragraph{Case $T_{0:H}$.} Under \cref{assumption:transitions}, the conditioned
  trajectory probability $\Prob(T_{0:H} \mid P)$, which is itself a random variable
  through conditioning on $P$, is a product of independent random variables defined by
  \begin{align}
    \Prob(T_{0:H} \mid P) &= \prod_{h=0}^{H-1}\pi(A_h \mid S_h)P(S_{h+1} \mid S_h, A_h) \\
    &= P(S_1 \mid S_0, A_0) \pi(A_0 \mid S_0)\prod_{h=1}^{H-1}\pi(A_h \mid S_h)P(S_{h+1} \mid S_h,
    A_h). \\
    &= P(S_1 \mid S_0, A_0)\pi(A_0 \mid S_0) \Prob(T_{1:H} \mid P).
  \end{align}
  Note that each transition probability in $\Prob(T_{0:H} \mid P)$ is distinct by
  \cref{assumption:acyclic}. Additionally, for any $h>0$ the action probability
  $\pi(A_h \mid S_h)$ is \emph{conditionally independent} of $P(S_h \mid S_{h-1},
  A_{h-1})$ given $S_h$. Then, for arbitrary $S_0=s$, $A_0 = a$ and $S_1 = s'$, we have
  that $P(s' \mid s, a)$ is conditionally independent of $\Prob(T_{1:H-1}\mid P)$. Since
  $V^\pi(S_1 \mid S_1 = s')$ is a function of $\Prob(T_{1:H}\mid P)$, then it is also
  conditionally independent of $P(s' \mid s, a)$.

  \paragraph{Case $T_{H:\infty}$.} The Lemma is trivially satisfied since both the
  transition probability and the values become constants: we have $P(s_T \mid s_T, a) =
  1$ and $V^\pi(s_T) = 0$.

  Combining both results, we have that $P$ and $V^\pi$ are conditionally independent for
  any arbitrary infinite trajectory, which completes the proof.
\end{proof}
\begin{lemma}
  \label{lemma:indep_property}
  Define an arbitrary $\mu_0 \in \mathcal{P}_B(\R)^\mathcal{S}$. Then, under
  \cref{assumption:transitions,assumption:acyclic,assumption:truncated}, the sequence
  $\set{\mu_k}_{k=0}^\infty$ defined by $\mu_{k+1} = \mathcal{T}^\pi\mu_k$ is such that
  for all $k\geq0$ the random variable $V_k(s') \sim \mu_k(s')$ is conditionally
  independent of $P^\pi(s' \mid s)$  given $s' \in \mathcal{S}$.
\end{lemma}
\begin{proof}
  Intuitively, by definition of the operator $\mathcal{T}^\pi$, the random variable
  $V_k(s')$ is the result of summing rewards starting from state $s'$, following the
  random dynamics $P^\pi$ for $k$ steps and then bootstraping with $V_0$. That is,
  applying $\mathcal{T}^\pi$ $k$ times results in the random variable
  \begin{equation}
    \label{eq:v_k}
    V_k(s') = \E_{T_{0:k}}\bracket{\sum_{h=0}^{k-1} \gamma^h R_h + \gamma^k V_0(S_k) \middle| S_0=s', P},
  \end{equation}
  where $T_{0:k}$ is a random state trajectory $\set{S_0, S_1, ..., S_k}$ starting from
  $S_0 = s'$ and following the random dynamics $P$. Under
  \cref{assumption:transitions,assumption:acyclic,assumption:truncated}, $T_{0:k}$ is a
  sequence of unique states (with the exception of the absorbing terminal state), which
  means that the probability of returning to $s'$ are zero. Finally, since $V_k(s')$ is
  an expectation conditioned on the initial state being $s'$, it follows that $V_k(s')$
  is conditionally independent of $P^\pi(s' \mid s)$ given $s'$.
\end{proof}

\begin{lemma}
  \label{lemma:bounded_support}
  If the value distribution function $\mu$ has bounded support, then $\mathcal{T}^\pi\mu$ also
  has bounded support.
\end{lemma}
\begin{proof}
  From bounded rewards on $[r_{\min}, r_{\max}]$, then we denote by
  $\mathcal{P}_B(\R)^\mathcal{S}$ the space of value distributions bounded on $[v_{\min},
  v_{\max}]$, where $v_{\min} = r_{\min} / (1 - \gamma)$ and $v_{\max} = r_{\max} / (1 - \gamma)$.

  Given arbitrary $\mu \in \mathcal{P}_B(\R)^\mathcal{S}$, let $v(s)$ be a realization of
  $\mu(s)$ for any $s \in \mathcal{S}$. Then, $\sum_a \pi(a \mid s)r(s,a) + \gamma \sum_{a, s'}\pi(a
  \mid s)P(s' \mid s, a) v(s')$ is an instantiation of $(\mathcal{T}^\pi \mu)(s)$ for any $s \in
  \mathcal{S}$. We have:
  \begin{align}
    \mathbb{P}\paren{\paren{\mathcal{T}^\pi \mu}(s) \leq v_{\max}} &= \mathbb{P} \paren{
      \sum_a \pi(a \mid s)r(s,a) + \gamma \sum_{a, s'}\pi(a \mid s)P(s' \mid s, a) v(s') \leq v_{\max}
    },
    \intertext{}
    &= \mathbb{P} \paren{
      \gamma \sum_{a, s'}\pi(a \mid s)P(s' \mid s, a) v(s') \leq v_{\max} - \sum_a \pi(a \mid s)r(s,a)
    }.
    \intertext{Since $\sum_a \pi(a \mid s)r(s,a) \leq r_{\max}$, then}
    &\geq \mathbb{P} \paren{
      \gamma \sum_{a, s'}\pi(a \mid s)P(s' \mid s, a) v(s') \leq v_{\max} - r_{\max} 
    }.
    \intertext{By definition of $v_{\max}$}
    &\geq \mathbb{P} \paren{
      \sum_{a, s'}\pi(a \mid s)P(s' \mid s, a) v(s') \leq v_{\max}
    }.
    \intertext{Finally, since $v(s') \leq v_{\max}$ for any $s' \in \mathcal{S}$, then}
    &=1.
  \end{align}
  Under the same logic, we can similarly show that $\mathbb{P}\paren{\paren{\mathcal{T}^\pi
  \mu}(s) \geq v_{\min}} =  1$, such that $\mathbb{P}\paren{\paren{\mathcal{T}^\pi \mu}(s) \in
  [v_{\min}, v_{\max}]} =  1$ for any $s \in \mathcal{S}$.
\end{proof}

\section{Implementation Details}
\label{app:implementation_details}
\subsection{Quantile Huber Loss}
\label{subsec:quantile_huber_loss}
We adopt the quantile Huber loss from \citet{dabney_distributional_2018} in order to train the
distributional critic. The Huber loss is given by
\begin{equation}
  \mathcal{L}_{\kappa}(u) = \begin{cases}
    \frac{1}{2}u^2, & \text{if } \abs{u} \leq \kappa \\
    \kappa(\abs{u} - \frac{1}{2}\kappa) & \text{otherwise}
  \end{cases},
\end{equation}
and the quantile Huber loss is defined by
\begin{equation}
  \rho^{\kappa}_\tau(u) = \abs{\tau - \delta(u < 0)}\mathcal{L}_\kappa(u).
\end{equation}
For $\kappa = 0$, we recover the standard quantile regression loss, which is not smooth as $u \to
0$. In all our experiments we fix $\kappa=1$ and to simplify notation define $\rho^{1}_\tau =
\rho_\tau$.

\subsection{EQR-SAC Algorithm}
\begin{algorithm}[t]
   \caption{Epistemic Quantile-Regression with Soft Actor-Critic (EQR-SAC)}
   \label{algorithm:our_algorithm_deep_rl}
\begin{algorithmic}[1]
  \STATE Initialize policy $\pi_{\phi}$, MDP ensemble $\Gamma_\psi$, quantile critic $q$,
  environment dataset $\mathcal{D}$, model dataset $\mathcal{D}_{\text{model}}$, utility function
  $f$.

  \STATE Warm-up $\mathcal{D}$ with rollouts under $\pi_\phi$
  \STATE global step $\leftarrow 0$
  \FOR{episode $t=0, \dots, T-1$}
    \FOR{$E$ steps}
      \IF{global step \% $F == 0$}
        \STATE Train $\Gamma_\psi$ on $\mathcal{D}$ via maximum likelihood
        \FOR{each MDP dynamics in $\Gamma_\psi$}
          \FOR{$L$ model rollouts}
            \STATE Perform $k$-step rollouts starting from $s \sim \mathcal{D}$; add to $\mathcal{D}_{\text{model}}$\label{line:model_rollouts}
          \ENDFOR
        \ENDFOR
      \ENDIF
      \STATE Take action in environment according to $\pi_{\phi}$; add to $\mathcal{D}$
      \FOR{$G$ gradient updates}
        \STATE Update $\set{q_i}_{i=1}^{m}$ with mini-batches from $\mathcal{D}_{\text{model}}$, via SGD on \cref{eq:critic_loss} \label{line:q_update}
        \STATE Update $\pi_\phi$ with mini-batches from $\mathcal{D}_{\text{model}}$, via SGD on \cref{eq:actor_loss} \label{line:pi_update}
      \ENDFOR
    \ENDFOR
    \STATE global step $\leftarrow$ global step $+ 1$
  \ENDFOR
\end{algorithmic}
\end{algorithm}

A detailed execution flow for training an EQR-SAC agent is presented in
\cref{algorithm:our_algorithm_deep_rl}. Further implementation details are now provided.

\paragraph{Model learning.} We use the \texttt{mbrl-lib} Python library from
\citet{pineda_mbrl-lib_2021} to train $N$ neural networks (Line 7). Our default architecture
consists of four fully-connected layers with 200 neurons each (for the Quadruped environments we use
400 neurons to accomodate the larger state space). The networks predict delta states, $(s' - s)$,
and receives as inputs normalizes state-action pairs. The normalization statistics are updated each
time we train the model and are based on the training dataset $\mathcal{D}$. We use the default
initialization that samples weights from a truncated Gaussian distribution, but we increase by a
factor of 2 the standard deviation of the sampling distribution.

\paragraph{Capacity of $\mathcal{D}_{\text{model}}$.} The capacity of the model buffer is computed
as $k \times L \times F \times N \times \Delta$, where $\Delta$ is the number of model updates we
want to retain data in the buffer. That is, the buffer is filled only with data from the latest
$\Delta$ rounds of model training and data collection (Lines 6-10).

\paragraph{Critic Loss.} The distributional critic is updated in Line 13, for which we use the loss
function \cref{eq:critic_loss}. To approximate the target quantiles \cref{eq:target_quantile}, we
use the learned generative model and the policy to generate transition tuples $(r, s', a')$. More
specifically, each $(s, a)$ pair in a mini-batch from $\mathcal{D}_{\text{model}}$ is repeated $X$
times and passed through every member of the ensemble of dynamics, thus generating $n$ batches of
$X$ predictions $(r, s')$. Then, every $s'$ prediction is repeated $Y$ times and passed through
$\pi_\phi$, thus obtaining $XY$ next state-action pairs $(s', a')$. This generated data is finally
used in \cref{eq:target_quantile} to better approximate the expectation. In our experiments we use
$X = Y$ and keep their product as a hyperparameter controlling the total amount of samples we use to
approximate the expectation.

\paragraph{Reference Implementation.} We use a single codebase for all experiments and share
architecture components amongst baselines whenever possible. The execution of experiments follows the workflow of \cref{algorithm:our_algorithm_deep_rl}. The SAC base
implementation follows the open-source repository
\url{https://github.com/pranz24/pytorch-soft-actor-critic} and we allow for either model-free or
model-based data buffers for the agent's updates, as done in \texttt{mbrl-lib}.

\raggedbottom

\section{Hyperparameters}
\label{app:hparams}
\renewcommand{\arraystretch}{1.2}
\begin{table}[H]
\caption{Hyperparameters for DeepMind Control Suite. In red, we highlight the only deviations of the base hyperparameters across all environments and baselines.}
\label{tab:hparam}
\begin{center}
\begin{tabular}{|c|c|}
\toprule
\textbf{Name}  & \textbf{Value} \\
\midrule
\multicolumn{2}{|c|}{\textbf{General}} \\
\midrule
$T$ - \# episodes & $250$\\
$E$ - steps per episode & $10^3$ \\ 
Replay buffer $\mathcal{D}$ capacity & $10^5$ \\
Warm-up steps (under initial policy) & $5 \times 10^3$ \\
\midrule
\multicolumn{2}{|c|}{\textbf{SAC}} \\
\midrule
$G$ - \# gradient steps & $10$\\
Batch size & $256$ \\
Auto-tuning of entropy coefficient $\alpha$? & Yes  \\
Target entropy & $-\text{dim}(\mathcal{A})$ \\
Actor MLP network & 2 hidden layers - 128 neurons - Tanh activations \\
Critic MLP network & 2 hidden layers - 256 neurons - Tanh activations \\
Actor/Critic learning rate & $3\times 10^{-4}$ \\
\midrule
\multicolumn{2}{|c|}{\textbf{Dynamics Model}} \\
\midrule
$n$ - ensemble size & $5$\\
$F$ - frequency of model training (\# steps) & $250$ \\
$L$ - \# model rollouts per step & $400$ \\
$k$ - rollout length & $5$ \\
$\Delta$ - \# Model updates to retain data & $10$ \\
Model buffer $\mathcal{D}_{\text{model}}$ capacity \textcolor{red}{(EQR-SAC)} & $L \times F \times k \times \Delta \textcolor{red}{(\times n)} = 5 \times 10^6 \textcolor{red}{(25 \times 10^6)}$
\\
Model MLP network \textcolor{red}{(quadruped)} & 4 layers - 200 \textcolor{red}{(400)} neurons - SiLU activations
\\
Learning rate & $1\times 10^{-3}$\\
\midrule
\multicolumn{2}{|c|}{\textbf{Quantile Network}} \\
\midrule
$m$ - \# quantiles & $51$ \\
\# $(s', a')$ samples (EQR-SAC only) & $25$ \\
\bottomrule
\end{tabular}
\end{center}
\end{table}

\section{DM Control Learning Curves}
\label{app:dm_control_curves}
\begin{figure}[H]
  \includegraphics{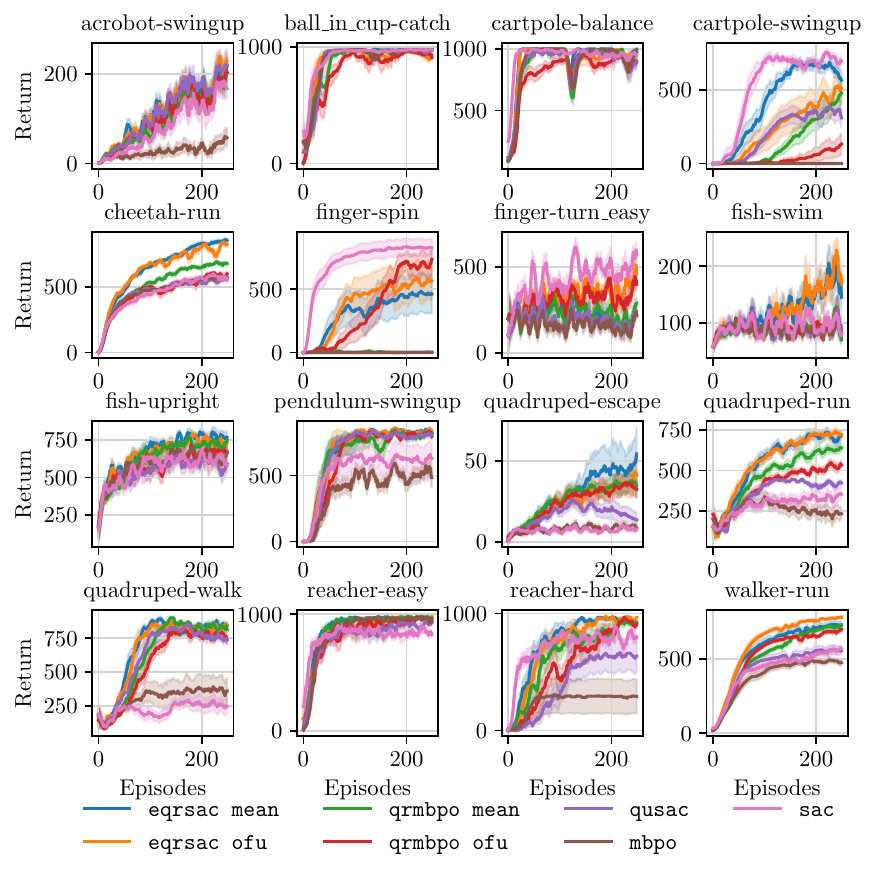}
  \caption{Individual learning curves of DMC benchmark.}
  \label{fig:full_dm_control}
\end{figure}

\pagebreak
\section{DM Control Final Scores}
\label{app:dm_control_scores}
\renewcommand{\arraystretch}{2.0}
\begin{table}[!ht]
\caption{Scores in DMC benchmark after 250 episodes (or 250K environment steps). For each
environment, we report the mean and standard error scores over 10 random seeds. We \textbf{bold} the
highest final mean score per environment.}
\label{tab:dm_control_scores}
\centerline{
\resizebox{1.0\textwidth}{!}{%
\begin{tabular}{|c|ccccccc|}
\toprule
Environment & \texttt{sac} & \texttt{mbpo} & \texttt{qrmbpo mean} &\texttt{qrmbpo ofu} &
\texttt{qusac} & \texttt{eqrsac mean} & \texttt{eqrsac ofu}\\
\midrule
acrobot-swingup & $167.7\pm22.4$ & $57.7\pm19.9$ & $202.3\pm18.7$ & $204.0\pm21.4$ &
$220.3\pm30.8$ & $217.8\pm18.5$ & $\bm{229.4}\pm17.4$ \\

ball-in-cup-catch & $972.6\pm3.3$ & $972.5\pm1.7$ & $974.4\pm2.3$ & $908.3\pm25.1$ &
$972.8\pm2.8$ & $\bm{977.2}\pm1.4$ & $928.7\pm18.5$\\

cartpole-balance-sparse & $949.5\pm18.7$ & $985.6\pm9.7$ & $977.0\pm22.2$ & $894.9\pm42.8$ &
$904.5\pm37.3$ & $\bm{997.8}\pm2.1$ & $968.0\pm15.2$\\

cartpole-swingup-sparse & $\bm{693.8}\pm27.2$ & $0.1\pm0.1$ & $476.2\pm72.6$ & $131.8\pm72.1$ &
$310.6\pm64.3$ & $566.4\pm54.4$ & $510.6\pm86.9$\\

cheetah-run & $551.6\pm22.6$ & $571.4\pm19.3$ & $679.1\pm10.7$ & $598.4\pm16.2$ & $567.4\pm16.8$ &
$\bm{854.0}\pm11.7$ & $820.0\pm18.4$\\

finger-spin & $\bm{827.5}\pm68.1$ & $0.1\pm0.1$ & $1.2\pm1.1$ & $734.4\pm89.7$ & $3.2\pm2.6$ &
$567.1\pm146.7$ & $461.9\pm154.1$\\

finger-turn-easy & $\bm{571.3}\pm31.3$ & $220.0\pm20.0$ & $289.8\pm34.1$ & $399.0\pm35.7$ &
$220.0\pm20.0$ & $221.3\pm19.9$ & $460.5\pm58.3$\\

fish-swim & $79.9\pm10.9$ & $80.6\pm10.0$ & $70.0\pm8.7$ & $91.9\pm13.3$ & $83.8\pm10.0$ &
$145.1\pm27.3$ & $\bm{168.3}\pm20.4$\\

fish-upright & $579.4\pm50.8$ & $660.5\pm59.5$ & $749.9\pm29.1$ & $671.4\pm32.2$ & $591.9\pm53.5$ &
$\bm{766.2}\pm45.3$ & $735.0\pm23.5$\\

pendulum-swingup & $631.0\pm111.7$ & $484.5\pm76.4$ & $819.2\pm17.2$ & $796.9\pm25.2$ &
$808.6\pm19.7$ & $\bm{834.4}\pm15.7$ & $833.7\pm16.0$\\

quadruped-escape & $8.0\pm1.2$ & $8.8\pm1.8$ & $34.5\pm7.1$ & $32.4\pm5.0$ & $13.7\pm4.3$
& $\bm{54.2}\pm16.7$ & $41.1\pm11.4$\\

quadruped-run & $352.0\pm36.4$ & $232.3\pm41.2$ & $638.5\pm26.0$ & $532.6\pm19.0$ & $421.9\pm16.8$
& $\bm{719.8}\pm19.0$ & $712.5\pm23.6$\\

quadruped-walk & $245.5\pm57.4$ & $360.1\pm95.1$ & $815.1\pm25.4$ & $739.4\pm38.2$ & $734.8\pm26.5$
& $844.3\pm26.8$ & $\bm{849.2}\pm17.1$\\

reacher-easy & $824.6\pm21.9$ & $474.3\pm20.8$ & $968.6\pm9.8$ & $959.1\pm13.2$ & $943.1\pm13.8$ &
$931.2\pm21.5$ & $\bm{977.9}\pm2.5$\\

reacher-hard & $797.5\pm38.8$ & $291.9\pm146.3$ & $921.8\pm22.2$ & $905.0\pm31.6$ & $635.0\pm139.5$
& $919.6\pm15.7$ & $\bm{965.3}\pm9.8$\\

walker-run & $568.9\pm19.1$ & $474.3\pm20.8$ & $725.5\pm10.8$ & $698.9\pm13.7$ & $553.8\pm30.9$ &
$727.4\pm24.3$ & $\bm{779.3}\pm7.9$\\

\bottomrule
\end{tabular}
}
}
\end{table}

\renewcommand{\arraystretch}{1.0}
\begin{table}[!ht]
\caption{Normalized inter-quartile mean scores in DMC benchmark after 100 and 250 episodes. For each
aggregation metric, we report the point estimate and the 95\% bootstrap confidence interval within
parentheses, following the methodology in \citet{agarwal_deep_2021}. We \textbf{bold} the highest
mean scores in each case.}
\label{tab:agg_dm_control_scores}
\centerline{
\begin{tabular}{lccc}
  \toprule
  Method & IQM-100 & IQM-250 \\ 
  \midrule
  \tikzcircle[fill=MPL-blue]{3pt} \texttt{eqrsac mean} & \textbf{0.63} (0.55, 0.72) & 0.73 (0.65,
  0.81) \\ 

  \tikzcircle[fill=MPL-orange]{3pt} \texttt{eqrsac ofu} & 0.61 (0.53, 0.69) & \textbf{0.76} (0.69,
  0.81)\\ 

  \tikzcircle[fill=MPL-green]{3pt} \texttt{qrmbpo mean} & 0.46 (0.37, 0.56) & 0.65 (0.56, 0.73) \\ 

  \tikzcircle[fill=MPL-red]{3pt} \texttt{qrmbpo ofu} & 0.46 (0.38, 0.55) & 0.64 (0.56, 0.69) \\ 

  \tikzcircle[fill=MPL-purple]{3pt} \texttt{qusac} & 0.41 (0.33, 0.50) & 0.51 (0.43, 0.58) \\

  \tikzcircle[fill=MPL-brown]{3pt} \texttt{mbpo} & 0.30 (0.23, 0.39) & 0.32 (0.24, 0.42)  \\ 

  \tikzcircle[fill=MPL-pink]{3pt} \texttt{sac} & 0.54 (0.46, 0.61) & 0.59 (0.52, 0.66) \\ 
  \bottomrule
\end{tabular}
}
\end{table}

\section{Ablations}
\label{app:ablations}
In this section, we present additional ablation studies on three salient hyperparameters of EQR-SAC:
the number of quantiles $(m)$, the rollout length $(k)$ and the number of model updates to retain
data $(\Delta)$. For all the experiments, we use the default hyperparameters in \cref{tab:hparam}
and only vary the hyperparameter of the corresponding ablation study.

\subsection{Number of Quantiles}
\begin{figure}[H]
  \includegraphics{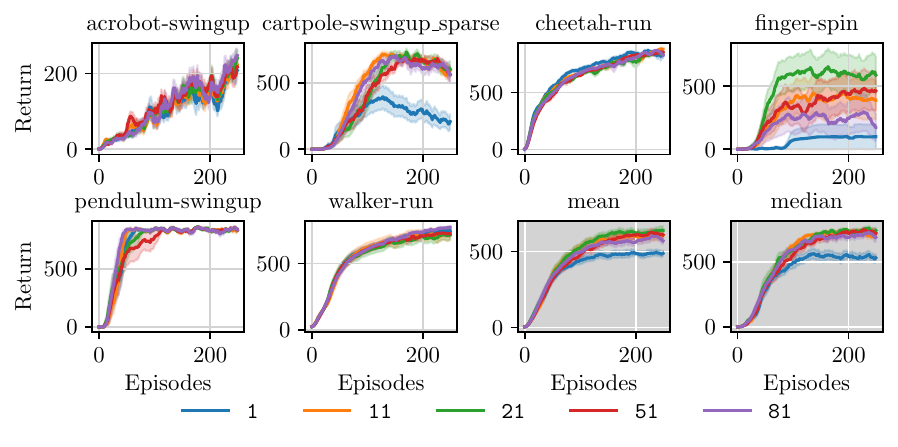}
  \includegraphics{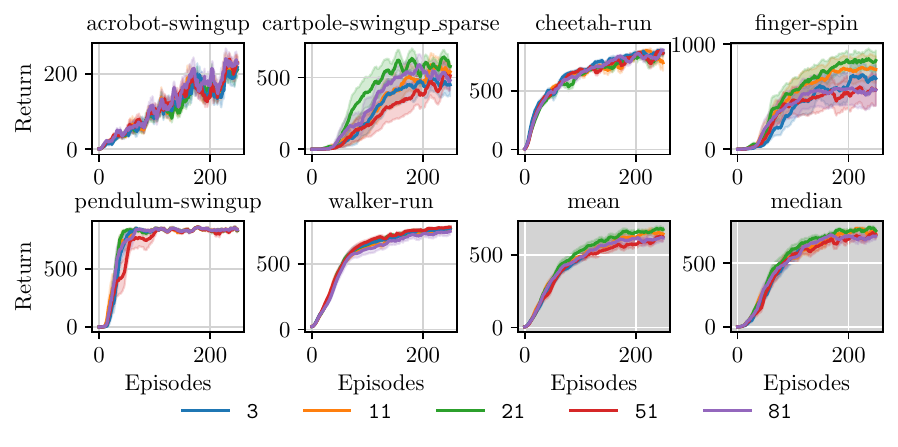}
  \caption{Number of quantiles ($m$) ablation study. \textbf{(Top)} EQR-SAC-mean. \textbf{(Bottom)}
   EQR-SAC-ofu. Note that for EQR-SAC-ofu we require $m>1$ in order to estimate the standard
   deviation of quantiles for the optimistic objective function of the actor, thus we select a
   minimum value of $m=3$ for this study.}
  \label{fig:num_quantiles_ablation}
\end{figure}
\subsection{Rollout Length}
\begin{figure}[H]
  \includegraphics{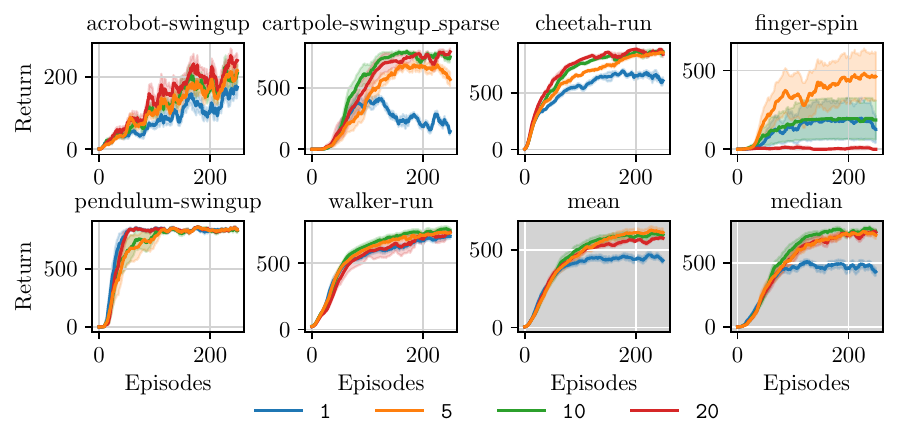}
  \caption{Rollout length ($k$) ablation study for EQR-SAC-mean.}
  \label{fig:rollout_length_ablation}
\end{figure}
\subsection{Number of Model Updates to Retain Data}
\begin{figure}[H]
  \includegraphics{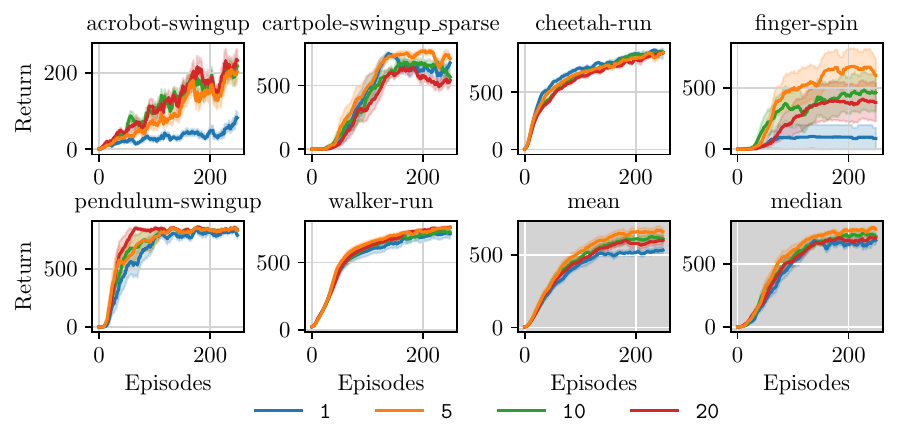}
  \caption{Number of model updates to retain data ($\Delta$) ablation study for EQR-SAC-mean.}
  \label{fig:retain_epochs_ablation}
\end{figure}

\section{Optimistic Policy Optimization}
\label{app:optimistic_opt}

In this section, we investigate the effect of using optimistic value estimates for policy
optimization. To conduct the study, we propose a simple variant of EQR-SAC, named EQR-SAC-$\tau$,
which uses as the actor's objective function the closest quantile level to a given target $\tau$.
For instance, in our experiments we use $m=21$ and target levels $\set{0.5, 0.7, 0.9}$, which
correspond to actual levels $\set{0.5, 0.69, 0.88}$.

\paragraph{Dense versus sparse rewards.} We first investigate how optimism affects performance
in environments with dense versus sparse rewards and present the results in
\cref{fig:dense_vs_sparse_optimism}. For environments with dense rewards (cheetah, walker) optimism
has little to no effect, while it results in largely different performance in envionments with
sparse rewards (reacher-hard, finger-spin). Even though we would expect optimism to be generally
helpful in all exploration tasks, our results indicate its effect is environment-dependent: the most
optimistic objective ($\tau=0.9$) performed worst in reacher-hard but obtained the best performance
in finger-spin; inversely, the least optimistic objective ($\tau=0.5$) performed the best in
reacher-hard, but worst in finger-spin.

\begin{figure}[H]
  \includegraphics{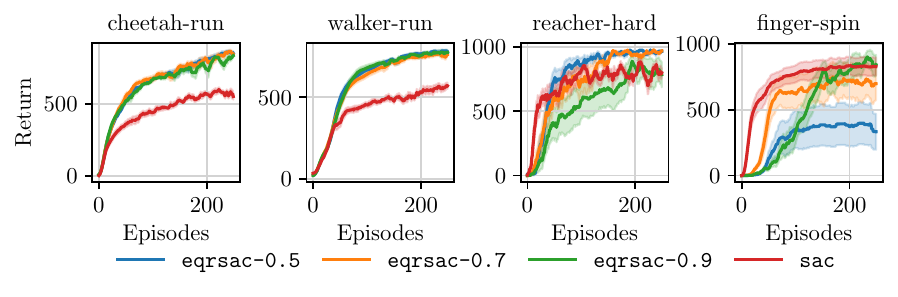}
  \caption{Evaluation of EQR-SAC-$\tau$ for different quantile levels. The two tasks on the left have dense rewards, while the other two have sparse rewards.}
  \label{fig:dense_vs_sparse_optimism}
\end{figure}

\paragraph{Action costs and sparse rewards.} In \cref{subsec:mountaincar}, we observe that the
combination of action costs and sparse rewards represents a pitfall for methods like SAC, especially
since the optimal policy must issue large actions to observe the reward. Meanwhile, the
quantile-based optimistic approaches performed best. In this experiment, we test the same setting in
two tasks with sparse rewards from the DeepMind Control suite, where we add an action cost
proportional to the squared norm of the action taken by the agent. Namely,
\begin{equation}
  \label{eq:action_cost}
  \text{\texttt{action\textunderscore cost}} = \rho \sum_{i=1}^{\abs{\mathcal{A}}} a_i^2
\end{equation} 
where $\rho$ is an environment specific base multiplier, $a_i$ is the $i$-th component of the action
vector and $\abs{\mathcal{A}}$ is the size of the action space. For \texttt{cartpole-swingup} we use
$\rho=0.001$ and for \texttt{pendulum} we use $\rho = 0.01$. The resuls in
\cref{fig:action_cost_ablation} show a similar degradation of performance for SAC. Unlike the
MountainCar experiments of \cref{subsec:mountaincar}, higher levels of optimism mostly resulted in
less sample-efficient learning.

\begin{figure}[H]
  \includegraphics{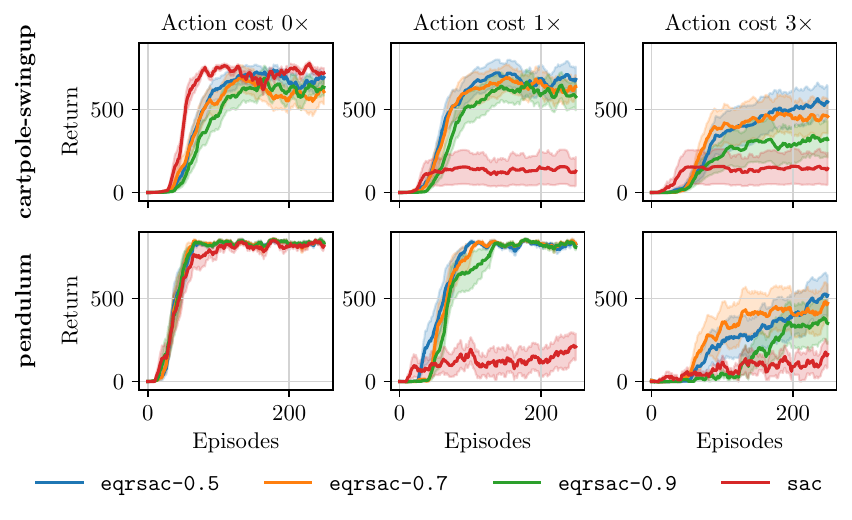}
  \caption{Evaluation of EQR-SAC-$\tau$ for different quantile levels and increasing action
  costs. The top row corresponds to the cartpole swingup task and the bottom row to the pendulum
  swingup. The action costs range from zero (left column), to a $3\times$ multiplier on
  \cref{eq:action_cost}.}
  \label{fig:action_cost_ablation}
\end{figure}


\clearpage
\typeout{}
\bibliography{references}

\end{document}